\documentclass{article}
\usepackage[final]{neurips_2020}
\usepackage[utf8]{inputenc} 
\usepackage[T1]{fontenc}    
\usepackage{url}       
\usepackage{booktabs}       
\usepackage{amsfonts}       
\usepackage{nicefrac}       
\usepackage{microtype}      

\usepackage{natbib}
\usepackage{amssymb,amsmath,amsthm}

\usepackage{xcolor}
\usepackage{hyperref} 
\usepackage{url} 
\usepackage{graphicx}
\definecolor{blue-violet}{rgb}{0.14, 0.17, 0.65}
\hypersetup{colorlinks, citecolor=blue-violet, linkcolor=red, urlcolor=blue}

\usepackage{tcolorbox}

\usepackage{graphicx} 
\usepackage{algorithm,algorithmic}

\usepackage{wrapfig}

\usepackage{mathtools}
\usepackage{color} 
\usepackage{enumerate,textpos}

\renewcommand{\a}{\mathrm{a}}
\renewcommand{\b}{\mathrm{b}}

\newcommand{\x}{\mathrm{x}}
\newcommand{\h}{\mathrm{h}}
\newcommand{\y}{\mathrm{y}}

\renewcommand{\u}{\mathrm{u}}
\renewcommand{\v}{\mathrm{v}}
\newcommand{\w}{\mathrm{w}}
\newcommand{\z}{\mathrm{z}}

\newcommand{\0}{\mathrm{0}}

\newcommand{\I}{\mathrm{I}}
\newcommand{\B}{\mathrm{B}}

\newcommand{\U}{\mathrm{U}}
\newcommand{\V}{\mathrm{V}}

\newcommand{\X}{\mathrm{X}}
\newcommand{\Y}{\mathrm{Y}}
\newcommand{\W}{\mathrm{W}}

\newcommand{\bb}{\mathbb}

\newcommand{\R}{\bb R}
\newcommand{\bI}{\bb I}
\newcommand{\N}{\bb N}

\newcommand{\E}{\bb E }

\newcommand{\cB}{\mathcal{B}}

\newcommand{\cD}{\mathcal{D}}

\newcommand{\cN}{\mathcal{N}}
\newcommand{\cS}{\mathcal{S}}
\newcommand{\cO}{\mathcal{O}}
\newcommand{\cR}{\mathcal{R}}
\newcommand{\cX}{\mathcal{X}}
\newcommand{\cY}{\mathcal{Y}}
\newcommand{\cW}{\mathcal{W}}

\newcommand{\trace}{\operatorname{Tr}}

\newcommand{\prob}{\mathbb{P}}

\newcommand{\abs}[1]{\left|#1\right|}
\newcommand{\relu}[1]{\sigma\left(#1\right)}
\renewcommand{\ln}[1]{\operatorname{ln}\left(#1\right)}

\theoremstyle{definition}
\newtheorem{theorem}{Theorem}[section]
\newtheorem{lemma}[theorem]{Lemma}

\newtheorem{assumption}{Assumption}

\title{On Convergence and Generalization\\ of Dropout Training}

\author{
  Poorya Mianjy\\
  Department of Computer Science\\
  Johns Hopkins University\\
  \texttt{mianjy@jhu.edu} 
  \And 
  Raman Arora\\
  Department of Computer Science\\
  Johns Hopkins University\\
  \texttt{arora@cs.jhu.edu}
}

\begin{document}

\maketitle

\begin{abstract}
We study dropout in two-layer neural networks with rectified linear unit (ReLU) activations. Under mild overparametrization and assuming that the limiting kernel can separate the data distribution with a positive margin, we show that dropout training with logistic loss achieves $\epsilon$-suboptimality in test error in $O(1/\epsilon)$ iterations.
\end{abstract}

\section{Introduction}\label{sec:intro}

Dropout is an algorithmic regularization approach that endows deep learning models with excellent generalization ability despite the non-convex nature of the underlying learning problem and the capacity of modern over-parameterized models to over-fit. Introduced by~\cite{hinton2012improving,srivastava2014dropout}, dropout involves randomly pruning the network at every iteration of backpropagation by turning off a random subset of hidden nodes. Like many popular algorithmic approaches that emerged as heuristics from practitioners with deep insight into the learning problem, dropout, while extremely successful in practice, lacks a strong theoretical justification, especially from a computational learning theoretic perspective. 

Dropout has been successful in several application areas including computer vision~\citep{szegedy2015going}, natural language processing~\citep{merity2017regularizing}, and speech recognition~\citep{dahl2013improving}. Motivated by explaining the empirical success of dropout, and inspired by simple, intuitive nature of the method, numerous works in recent years have focused on understanding its theoretical underpinnings~\citep{baldi2013understanding,wager2013dropout,helmbold2015inductive,gal2016dropout,wei2020implicit}. Most of these works, however, steer clear from the algorithmic and computational learning aspects of dropout. More precisely, none of the prior work, to the best of our knowledge, yields insights into the runtime of learning using dropout on non-linear neural networks. In this paper, we initiate a study into the iteration complexity of dropout training for achieving $\epsilon$-suboptimality on true error -- the misclassification error with respect to the underlying population -- in two-layer neural networks with ReLU activations.

We leverage recent advances in the theory of deep learning in over-parameterized settings with extremely (or infinitely) wide networks~\citep{jacot2018neural,lee2019wide}. Analyzing two-layer ReLU networks in such a regime has led to a series of exciting results recently establishing that gradient descent (GD) or stochastic gradient descent (SGD) can successfully minimize the empirical error and the true error~\citep{li2018learning,du2018gradient,daniely2017sgd,zou2018stochastic,allen2019learning,song2019quadratic,arora2019fine,cao2019generalization,oymak2020towards}. In a related line of research, several works attribute generalization in over-parametrized settings to the implicit inductive bias of optimization algorithms (through the geometry of local search methods)~\citep{neyshabur2017geometry}. However, many real-world state-of-the-art systems employ various explicit regularizers, and there is growing evidence that implicit bias may be unable to explain generalization even in a simpler setting of stochastic convex optimization~\citep{dauber2020can}. Here, we extend convergence guarantees and generalization bounds for GD-based methods with explicit regularization due to dropout. We show that the key insights from analysis of GD-based methods in over-parameterized settings carry over~to~dropout~training.

We summarize our key contributions as follows. 
\begin{enumerate}
\item We give precise non-asymptotic convergence rates for achieving $\epsilon$-subotimality in the test error via dropout training in two-layer ReLU networks.
\item We show that dropout training implicitly compresses the network. In particular, we show that there exists a sub-network, i.e., one of the iterates of dropout training, that can generalize as well as any complete network. 
\item This work contributes to a growing body of work geared toward a theoretical understanding of GD-based methods for regularized risk minimization in over-parameterized settings.
\end{enumerate}
 
The rest of the paper is organized as follows. In Section~\ref{sec:related}, we survey the related work. In Section~\ref{sec:preliminaries}, we formally introduce the problem setup and dropout training, state the key assumptions, and introduce the notation. In Section~\ref{sec:results}, we give the main results of the paper. In Section~\ref{sec:proofs}, we present a sketch of the proofs of our main results -- the detailed proofs are deferred to the Appendix. We conclude the paper by providing empirical evidence for our theoretical results in Section~\ref{sec:exp}.
\section{Related Work}\label{sec:related}

Empirical success of dropout has inspired a series of works aimed at understanding its theoretical underpinnings. Most of these works have focused on explaining the algorithmic regularization due to dropout in terms of \emph{conventional} regularizers. Dropout training has been shown to be similar to \emph{weight decay} in linear regression~\citep{srivastava2014dropout}, in generalized linear models~\citep{wager2013dropout},  and in a PAC-Bayes setting~\citep{mcallester2013pac}.  These results have recently been extended to multi-variate regression models, {where dropout induces a \emph{nuclear norm} penalty in single hidden-layer linear networks~\citep{mianjy2018implicit} and deep linear networks~\citep{mianjy2019dropout}, and a \emph{weighted trace-norm} penalty in matrix completion~\citep{arora2020dropout}. }
In a recent work, \citet{wei2020implicit} characterize explicit regularization due to dropout in terms of the derivatives of the loss and the model, and argue for an implicit regularization effect that stems from the stochasticity in dropout updates.

A parallel strand of research has focused on bounding the generalization gap in dropout training, leveraging tools from uniform convergence. The early work of~\citet{wager2014altitude} showed that under a certain topic model assumption on the data, dropout in linear classification can improve the decay of the excess risk of the empirical risk minimizer. Assuming certain norm-bounds on the weights of the network, the works of \citet{wan2013regularization,zhai2018adaptive,gao2016dropout} showed that the Rademacher complexity of networks trained by dropout decreases with the dropout rate. Finally, \citet{arora2020dropout} showed further that the Rademacher complexity can be bounded merely in terms of the explicit regularizer induced by dropout.

Despite the crucial insights provided by the previous art, there is not much known about the non-asymptotic convergence behaviour of dropout training in the literature. A very recent work by~\cite{senen2020almost} shows for deep neural networks with polynomially bounded activations with continuous derivatives, under squared loss, that the network weights converge to a stationary set of system of ODEs. 
In contrast, our results  leverages over-parameterization in two-layer networks with non-differentiable ReLU activations, works with logistic loss, and establishes $\epsilon$-suboptimality in the true misclassification error.

Our results are inspired by the recent advances in over-parameterized settings. A large body of literature has focused on deriving optimization theoretic guarantees for (S)GD in this setting. In particular,~\cite{li2018learning,du2018gradient} were among the first  to provide convergence rates for  empirical risk minimization using GD. Several subsequent works extended those results beyond two-layers, for smooth activation functions~\citep{du2018gradient1}, and general activation functions~\citep{allen2018convergence,zou2018stochastic} .

Learning theoretic aspects of GD-based methods have been studied for several important target concept classes. Under linear-separability assumption, via a compression scheme,  \citet{brutzkus2017sgd} showed that SGD can efficiently learn a two-layer ReLU network. \citet{li2018learning} further showed that SGD enjoys small generalization error on two-layer ReLU networks if the data follows a well-separated mixture of distributions. \citet{allen2019learning} showed generalization error bounds for SGD in two- and three-layer networks with smooth activations where the concept class has fewer parameters. \citet{arora2019fine} proved data-dependent generalization error bounds based on the neural tangent kernel by analyzing the Rademacher complexity of the class of networks~reachable~by~GD. 

When the data distribution can be well-classified in the \emph{random feature space} induced by the gradient of the network at initialization,~\cite{cao2019generalization} provide generalization guarantees for SGD in networks with arbitrary depth. \cite{nitanda2019refined} studied convergence of GD in two-layer networks with smooth activations, when the data distribution is further \emph{separable} in the infinite-width limit of the random feature space. \cite{ji2019polylogarithmic} adopted the same margin assumption and improved the convergence rate as well as the over-parameterization size for non-smooth ReLU activations. Here, we generalize the margin assumption in~\cite{nitanda2019refined} to take into account the randomness injected by dropout into the gradient of the network at initialization, or equivalently, the scale of the corresponding random feature. Our work is most closely related to and inspired by~\cite{ji2019polylogarithmic}; however, we analyze dropout training as opposed to plain SGD, give generalization bounds in expectation, and show the compression benefits of dropout training.

We emphasize that all of the results above focus on (S)GD in absence of any explicit regularization. We summarize a few papers that study regularization in the over-parameterized setting. The work of~\cite{wei2019regularization} showed that even simple explicit regularizers such as weight decay can indeed provably improve the sample complexity of training using GD in the Neural Tangent Kernel (NTK) regime, appealing to a margin-maximization argument in homogeneous networks. We also note the recent works by~\cite{li2019gradient} and~\cite{Hu2020Simple}, which studied the robustness of GD to noisy labels, with explicit regularization in forms of early stopping; and squared norm of the distance from initialization, respectively.

\section{Preliminaries and Notation}\label{sec:preliminaries}
Let $\cX\subseteq \R^{d}$ and $\cY = \{\pm 1 \}$ denote the input and label spaces, respectively. We assume that the data is jointly distributed according to an unknown distribution $\cD$ on $\cX \times \cY$. Given $T$ i.i.d. examples $\cS_T=\{ (\x_t,\y_t)\}_{t=1}^T \sim \cD^T$, the goal of learning is to find a hypothesis $f(\cdot;\Theta):\cX\to\R$, parameterized by $\Theta$, that has a small \emph{misclassification error} $\cR(\Theta):= \prob\{y f(\x;\Theta) < 0 \}$. Given a convex surrogate loss function $\ell: \R \to \R_{\geq 0}$, a common approach to the above learning problem is to solve the stochastic optimization problem $\min_\Theta L(\Theta) := \E_\cD[\ell(y f(\x;\Theta))]$. 

In this paper, we focus on logistic loss $\ell(z)=\log(1+e^{-z})$, which is one of the most popular loss functions for classification tasks. We consider two-layer ReLU networks of width $m$, parameterized by the ``weights'' $\Theta=(\W,\a)\in \R^{m\times d} \times \R^m$, computing the function $f(\cdot;\Theta) :\x \mapsto \frac1{\sqrt m}\a^\top\relu{\W \x}$. We initialize the network with $a_r\sim\text{Unif}(\{+1,-1\})$ and $\w_{r,1}\sim\cN(0,\I)$, for all hidden nodes $r\in[m]$. We then fix the top layer weights and train the hidden layer $\W$ using the dropout algorithm. {We denote the weight matrix at time $t$ by $\W_t$, and $\w_{r,t}$ represents its $r$-th column.} For the sake of simplicity of the presentation, we drop non-trainable arguments from all functions, e.g., we use $f(\cdot;\W)$ in lieu of~$f(\cdot;\Theta)$. 

Let $\B_t \in \R^{m \times m}, \ t \in [T]$, be a random diagonal matrix with diagonal entries drawn independently and identically from a Bernoulli distribution with parameter $q$, i.e., $b_{r,t}\sim \text{Bern}(q)$, where $b_{r,t}$ is the $r$-th diagonal entry of $\B_t$. At the $t$-th iterate, dropout entails a SGD step on (the parameters of) the sub-network $g(\W;\x,\B_t) = \frac{1}{\sqrt{m}}\a^\top\B_t \sigma(\W\x)$, yielding updates of the form $\W_{t+\frac12}\gets \W_t - \eta \nabla \ell(y_t g(\W_t;\x_t,\B_t))$. The iteration concludes with projecting the incoming weights -- i.e. rows of $\W_{t+\frac12}$ -- onto a pre-specified Euclidean norm ball. We note that such max-norm constraints are standard in the practice of deep learning, and has been a staple to dropout training since it was proposed in~\cite{srivastava2014dropout}\footnote{Quote from~\cite{srivastava2014dropout}: ``One particular form of regularization was found to be especially useful for dropout—
constraining the norm of the incoming weight vector at each hidden unit to be upper
bounded by a fixed constant $c$''}. Finally, at test time, the weights are multiplied by $q$ so as to make sure that the output at test time is on par with the expected output at training time. The pseudo-code for dropout training is given in Algorithm~\ref{alg:dropout}\footnote{In a popular variant that is used in machine learning frameworks such as PyTorch, known as inverted dropout, (inverse) scaling is applied at the training time instead of the test time. The inverted dropout is equivalent to the method we study here, and can be analyzed in a similar manner.
}.

\begin{algorithm}[t!]
\caption{\label{alg:dropout}Dropout in Two-Layer Networks}
\begin{algorithmic}[1]
\REQUIRE data $\cS_T=\{(\x_t,y_t)\}_{t=1}^{T}\sim \cD^T$; Bernoulli masks $\cB_T=\{ \B_t \}_{t=1}^{T}$; \\ dropout rate $1-q$; max-norm constraint parameter $c$; learning rate $\eta$
\STATE \emph{initialize:} $\w_{r,1} \sim \cN(\0, \I)$ and  $a_r\sim \text{Unif}(\{+1,-1\})$, $\ r\in [m]$
\FOR{$t = 1,\dots,T-1$}
\STATE \emph{forward:} $g(\W_t;\x_t,\B_t) = \frac{1}{\sqrt{m}}\a^\top \B_t \sigma(\W_t \x_t)$
\STATE \emph{backward:} $\nabla L_t(\W_t) = \nabla  \ell(y_t g(\W_t;\x_t,\B_t) = \ell'(y_t g(\W_t;\x_t,\B_t)) \cdot y_t \nabla g(\W_t;\x_t,\B_t)$
\STATE \emph{update:} $\W_{t+\frac{1}{2}} \gets \W_{t} - \eta \nabla L_t(\W_t)$
\STATE \emph{max-norm:} $\W_{t+1} \gets \Pi_c(\W_{t+\frac12})$
\ENDFOR
\ENSURE re-scale the weights as $\W_t \gets q \W_t$
\end{algorithmic}
\end{algorithm}

Our analysis is motivated by recent developments in understanding the dynamics of (S)GD in the so-called \emph{lazy regime}. Under certain initialization, learning rate, and network width requirements, these results show that the iterates of (S)GD tend to stay close to initialization; therefore, a first-order Taylor expansion of the $t$-th iterate around initialization, i.e. $f(\x;\W_t) \approx f(\x;\W_1) + \langle \nabla f(\x; \W_1), \W_t - \W_1\rangle$, can be used as a proxy to track the evolution of the network predictions~\citep{li2018learning,chizat2018lazy,du2018gradient,lee2019wide}. In other words, training in lazy regime reduces to finding a linear predictor in the reproducing kernel Hilbert space (RKHS) associated with the gradient of the network at initialization, $\nabla f(\cdot;\W_1)$. In this work, following~\cite{nitanda2019refined,ji2019polylogarithmic}, we assume that the data distribution is separable by a positive margin in the limiting RKHS: 
\begin{assumption}[$(q,\gamma)$-Margin]\label{ass:margin}
Let $\z\sim \cN(\0,\I_d)$ and $b \sim \text{Bern}(q)$ be a $d$-dimensional standard normal random vector, and a Bernoulli random variable with parameter $q$, respectively. There exists a \emph{margin parameter} $\gamma > 0$, and a linear transformation $\psi : \R^d \to \R^d$ satisfying A) $\E_\z[\|\psi(\z)\|^2] < \infty;$ B) $\|\psi(\z)\|_2\leq 1$ for all $\z\in \R^d;$ and C) $\E_{\z, b} [y \langle\psi(\z), b\x \bI[\z^\top \x \geq 0]\rangle] \geq \gamma$ for almost all $(\x,y) \sim \cD$.
\end{assumption}
The above assumption provides an infinite-width extension to the separability of data in the RKHS induced by $\nabla g(\W_1;\cdot,\B_1)$. To see that, define $\V := [\v_1, \ldots, \v_m]^\top \in \R^{m \times d}$, where $\v_r = \frac{1}{\sqrt{m}}a_r\psi(\w_{r,1})$ for all $r \in [m]$, satisfying $\|\V\|_F\leq 1$. For any given point $(\x,y)\in\cX\times \cY$, the margin attained by $\V$ is at least $y \langle \nabla g(\W_1;\x, \B_1), \V\rangle = \frac{1}{m}\sum_{r=1}^{m}  y \langle  \psi(\w_{r,1}), b_{r,1}\x \bI\{ \w_{r,1}^\top \x >0 \} \rangle$, which is a finite-width approximation of the quantity $\E [y \langle \psi(\z), b\x  \bI\{\z^\top\x > 0\}\rangle]$ in Assumption~\ref{ass:margin}.

We remark that when $q = 1$ (pure SGD -- no dropout), with probability one it holds that $b =1$, so that Assumption~\ref{ass:margin} boils down to that of~\cite{nitanda2019refined} and~\cite{ji2019polylogarithmic}. When $q < 1$, this assumption translates to a margin of $\gamma/q$ on the \emph{full} features $\nabla f(\cdot;\W_1)$, which is the appropriate scaling given that $\nabla f(\cdot; \W_1) = \frac{1}{q}\E_\B[\nabla g(\W_1;\cdot,\B)]$. Alternatively, dropout training eventually outputs a network with weights scaled down as $q \W_t$, which (in expectation) corresponds to the shrinkage caused by the Bernoulli mask in $b \x \bI\{\z^\top\x > 0\}$. Regardless, we note that \emph{our analysis can be carried over even without this scaling}; however, new polynomial factors of $1/q$ will be introduced in the required  width in our results in Section~\ref{sec:results}. 

\subsection{Notation}
We denote matrices, vectors, scalar variables and sets by Roman capital letters, Roman small letters, small letters, and script letters, respectively (e.g. $\Y$, $\y$, $y$, and $\cY$). The $r$-th entry of vector $\y$, and the $r$-th row of matrix $\Y$, are denoted by $y_i$ and $\y_i$, respectively. Furthermore, for a sequence of matrices $\W_t, t \in \N$, the $r$-th row of the $t$-th matrix is denoted by $\w_{r,t}$. Let $\bI$ denote the indicator of an event, i.e., $\bI\{y\in \cY\}$ is one if $y\in \cY$, and zero otherwise. For any integer $d$, we represent the set $\{ 1,\ldots,d \}$ by $[d]$. Let $\| \x \|$ represent the $\ell_2$-norm of vector $\x$; and $\| \X \|$, and $\| \X \|_F$ represent the operator norm, and the Frobenius norm of matrix $\X$, respectively. $\langle \cdot, \cdot \rangle$ represents the standard inner product, for vectors or matrices, where $\langle \X,\X' \rangle = \trace(\X^\top \X')$. For a matrix $\W\in \R^{m\times d}$, and a scalar $c > 0$, $\Pi_c(\W)$ projects the rows of $\W$ onto the Euclidean ball of radius $c$ with respect to the $\ell_2$-norm.

For any $t\in[T]$ and any $\W$, let $f_t(\W):=f(\x_t;\W)$ denote the network output given input $\x_t$, and let $g_t(\W):=g(\W;\x_t,\B_t)$ denote the corresponding output of the sub-network associated with the dropout pattern $\B_t$. Let $L_t(\W)=\ell(y_t g_t(\W))$ and $Q_t(\W)=-\ell'(y_t g_t(\W))$ be the associated instantaneous loss and its negative derivative. The partial derivative of $g_t$ with respect to the $r$-th hidden weight vector is given by $\frac{\partial g_t(\W)}{\partial \w_r}=\frac{1}{\sqrt{m}}a_r b_{r,t} \bI\{\w_r^\top\x_t\geq 0\}\x_t$. We denote the linearization of $g_t(\W)$ based on features at time $t$ by $g_t^{(k)}(\W):=\langle \nabla g_t(\W_k), \W  \rangle$; and its corresponding instantaneous loss and its negative derivative by $ L^{(k)}_t(\W):=\ell(y_t g_t^{(k)}(\W))$ and $Q^{(k)}_t(\W):=-\ell'(y_t g_t^{(k)}(\W))$, respectively. $Q$ plays an important role in deriving generalization bounds for dropout sub-networks $g(\W_t;\x,\B_t)$; it has been recently used in~\citep{cao2019generalization,ji2019polylogarithmic} for analyzing the convergence of SGD and bounding its generalization error.

We conclude this section by listing a few useful identities that are used throughout the paper. First, due to homogeneity of the ReLU, it holds that $g^{(t)}_t(\W_t) = \langle \nabla g_t(\W_t), \W_t \rangle = g_t(\W_t)$. Moreover, the norm of the network gradient, and the norm of the the gradient of the instantaneous loss can be upper-bounded as $\|\nabla g_t(\W)\|_F^2 = \sum_{r=1}^{m} \| \frac{\partial g_t(\W)}{\partial \w_r} \|^2 \leq \frac{\|\B_t\|_F^2}{m} \leq 1,$ and $\|\nabla  L_t(\W)\|_F =  -\ell'(y_t g_t(\W)) \| y_t \nabla g_t(\W)\|_F \leq Q_t(\W),$ respectively. Finally, the logistic loss satisfies $|\ell'(z)| \leq \ell(z)$, so that $Q_t(\W)\leq  L_t(\W)$.

\section{Main Results}\label{sec:results}
We begin with a simple observation that given the random initialization scheme in Algorithm~\ref{alg:dropout}, the $\ell_2$-norm of the rows of $\W_1$ are expected to be concentrated around $\sqrt{d}$. In fact, using Gaussian concentration inequality (Theorem~\ref{thm:gaussian_concentration} in the appendix), it holds with probability at least $1-1/m$, uniformly for all $r\in [m]$, that $\|\w_{r,1}\|\leq \sqrt{d} +  2\sqrt{\ln{m}}$. For the sake of the simplicity of the presentation, we assume that the event $\max_{r\in[m]}\|\w_{r,1}\|\leq 2\sqrt{\ln{m}}$ holds through the run of dropout training. Alternatively, we can re-initialize the weights until this condition is satisfied, or multiply the probability of success in our theorems by a factor of $1-1/m$.

Our first result establishes that the true misclassification error of dropout training  vanishes as $\tilde\cO(1/T)$. 

\begin{theorem}[Learning with Dropout]\label{thm:overall}
Let $c=\sqrt{d} + \max\{\frac{1}{14\gamma^2}, 2\sqrt{\ln{m}}\} + 1$ and $\lambda :=  5\gamma^{-1}\ln{2\eta T}+\sqrt{44 \gamma^{-2}\ln{24 \eta c\sqrt{m} T^2}}$. Under Assumption~\ref{ass:margin}, for any learning rate $\eta\in (0,\ln{2}]$ and any network of width satisfying $m \geq 2401 \gamma^{-6} \lambda^2$, with probability one over the randomization due to dropout, we have that
\begin{equation*}
\min_{t\in[T]}\E[ \cR(q\W_t) ] \leq \frac{1}{T}\sum_{t=1}^{T}\E[ \cR(q\W_t) ] \leq \frac{4 \lambda^2}{\eta T} = \cO\left(\frac{\ln{T}^2 + \ln{m d T}}{T}\right),
\end{equation*}
where the expectation is with respect to the initialization and the training samples.
\end{theorem}

Theorem~\ref{thm:overall} shows that dropout successfully trains the complete network $f(\cdot;\W_t)$. Perhaps more interestingly, our next result shows that dropout successfully trains a potentially significantly narrower sub-network $g(\W_t;\cdot,\B_t)$. For this purpose, denote the misclassification error due to a network with weights $\W$ given a Bernoulli mask $\B$ as follows\\ 
\begin{equation*}
\cR(\W;\B) := \prob\{ y g(\W;\x,\B) < 0 \}.
\end{equation*}
Then the following result holds for the misclassification error of the iterates of dropout training. 
\begin{theorem}[Compression with  Dropout]\label{thm:sub-network}
Under the setting of Theorem~\ref{thm:overall}, with probability at least $1-\delta$ over initialization, the training data, and the randomization due to dropout, we have that
\begin{equation*}
\min_{t\in[T]} \cR(\W_t;\B_t) \leq \frac{1}{T}\sum_{t=1}^{T} \cR(W_t;\B_t) \leq \frac{12 \lambda^2}{\eta T} + \frac{6\ln{1/\delta}}{T} = \cO\left(\frac{\ln{m T}}{T}\right). 
\end{equation*}
\end{theorem}
A few remarks are in order.

Theorem~\ref{thm:overall} gives a generalization error bound in expectation. A technical challenge here stems from the unboundedness of the logistic loss. In our analysis, the max-norm constraint in Algorithm~\ref{alg:dropout} is essential to guarantee that the logistic loss remains bounded through the run of the algorithm, thereby controlling the loss of the iterates in expectation. However, akin to analysis of SGD in the lazy regime, the iterates of dropout training are not likely to leave a small proximity of the initialization whatsoever. Therefore, for the particular choice of $c$ in the above theorems, the max-norm projections in Algorithm~\ref{alg:dropout} will be virtually inactive for a typical run.

The expected width of the sub-networks in Theorem~\ref{thm:sub-network} is only $qm$. Using Hoeffding's inequality and a union bound argument, for any $\delta\in(0,1)$, with probability at least $1-\delta$, it holds for all $t\in [T]$ that $g(\W_t; \x, \B_t)$ has at most $qm + \sqrt{2m\ln{T/\delta}}$ active hidden neurons. That is, in a typical run of the dropout training, with high probability, there exists a sub-network of width $\approx qm + \tilde\cO(\sqrt{m})$ whose generalization error is no larger than $\tilde\cO(1/T)$. In Section~\ref{sec:exp}, we further provide empirical evidence to verify this compression result. We note that dropout has long been considered as a means of network compression, improving post-hoc pruning~\citep{gomez2019learning}, in Bayesian settings~\citep{molchanov2017variational}, and in connection with the Lottery Ticket Hypothesis~\citep{frankle2018the}. However, we are not aware of any theoretical result supporting that claim prior to our work.  

Finally, the sub-optimality results in both Theorem~\ref{thm:overall} and Theorem~\ref{thm:sub-network} are agnostic to the dropout rate $1-q$. This is precisely due to the $(q,\gamma)$-Margin assumption: if it holds, then so does $(q',\gamma)$-Margin for any $q'\in[q,1]$. That is, these theorems hold for \emph{any} dropout rate not exceeding $1-q$. Therefore, in light of the remark above, larger admissible dropout rates are preferable since they result in higher compression rates, while enjoying the same generalization error guarantees.
\section{Proofs}\label{sec:proofs}
We begin by bounding $\E_{\cS_t}[\cR(q\W_t)]$, the expected population error of the iterates, in terms of $\E_{\cS_t,\cB_t}[L_t(\W_t)]$, the expected instantaneous loss of the random sub-networks. In particular, using simple arguments including the smoothing property, the fact that $\W_t$ is independent from $(\x_t,y_t)$ given $\cS_{t-1}$, and that logistic loss upper-bounds the zero-one loss, it is easy to bound the expected population risk as $\E_{\cS_t}[\cR(q\W_t)] \leq \E_{\cS_t}[\ell(y_t f(\x_t; q\W_t))]$. Furthermore, using Jensen's inequality,~we have that $\ell(y_t f_t(q\W_t))\leq \E_{\B_t}[L_t(\W_t)].$ The upper bound then follows from these two inequalities. In the following, we present the main ideas in bounding the average instantaneous loss of the iterates.

Under Algorithm~\ref{alg:dropout}, dropout iterates are guaranteed to remain in the set $\cW_c:=\{ \W\in \R^{m\times d}: \ \|\w_r\| \leq c \}$. Using this property and the dropout update rule, we track the distance of consecutive iterates $(\W_{t+1},\W_t)$ from any competitor $\U \in \cW_c$, which leads to the following upper bound on the average instantaneous loss of iterates.

\begin{lemma}\label{lem:lyapunov}
Let $\W_1,\ldots,\W_{T}$ be the sequence of dropout iterates with a learning rate satisfying $\eta \leq \ln{2}$. Then, it holds for any $\U\in \cW_c$ that 
\begin{equation}\label{eq:regret}
\frac{1}{T}\sum_{t=1}^{T}L_t(\W_t) \leq \frac{\|\W_1 - \U \|_F^2}{\eta T} + \frac{2}{T}\sum_{t=1}^{T} L^{(t)}_t(\U).
\end{equation}
\end{lemma}
Note that the upper bound in Equation~\eqref{eq:regret} holds for \emph{any} competitor $\U \in \cW_c$; however, we seek to minimize the upper-bound on the right hand side of Equation~\eqref{eq:regret} by finding a \emph{sweet spot} that maintains a trade-off between 1) the distance from initialization, and 2) the average instantaneous loss for the linearized models. Following~\cite{ji2019polylogarithmic}, we represent such a competitor as an interpolation between the initial weights $\W_1$ and the \emph{max-margin} competitor $\V$, i.e. $\U:=\W_1+\lambda \V$, where $\lambda$ is the trade-off parameter. Recall that $\V := [\v_1,\cdots,\v_m] \in \R^{d \times m}$, where $\v_r = \frac{1}{\sqrt{m}}a_r\psi(\w_{r,1})$ for any $r \in [m]$, and $\psi$ is given by assumption~\ref{ass:margin} . Thus, the first term on the right hand side above can be conveniently bounded as $\frac{\lambda^2}{\eta T}$; Lemma~\ref{lem:loss_ub} bounds~the~second~term~as~follows.
\begin{lemma}\label{lem:loss_ub}
Under the setting of Theorem~\ref{thm:overall}, it holds with probability at least $1-\delta$ simultaneously for all iterates $t\in[T]$ that 1) $\|\w_{r,t}-\w_{r,1}\|\leq \frac{7 \lambda}{2\gamma\sqrt{m}}$, for all $r\in[m]$; and 2) $L^{(t)}_t(\U) \leq \frac{\lambda^2}{2\eta T}.$
\end{lemma}
Therefore, left hand side of Equation~\eqref{eq:regret}, i.e., the average instantaneous loss of the iterates, can be bounded with high probability as $\frac{2\lambda^2}{\eta T}$. To get the bound in expectation, as presented in Theorem~\ref{thm:overall}, we need to control $ L^{(t)}_t(\U)$ in worst-case scenario. We take advantage of the max-norm constraints in Algorithm~\ref{alg:dropout}, and show in Lemma~\ref{lem:worst_case} that with probability one all iterates satisfy $L^{(t)}_t(\U) \leq \frac{c\sqrt{m}}{\ln{2}}+1$. The proof of Theorem~\ref{thm:overall} then follows from carefully choosing $\delta$, the confidence parameter. To prove the compression result in Theorem~\ref{thm:sub-network}, we use the fact that the zero-one loss can be bounded in terms of the negative derivative of the logistic loss~\citep{cao2019generalization}. Therefore, we can bound $\cR(\W_t;\B_t)$, the population risk of the sub-networks, in terms of $Q(\W_t;\B_t) = \E_\cD[-\ell'(y_t g_t(\W_t))]$. Following~\cite{ji2019polylogarithmic}, the proof of Theorem~\ref{thm:sub-network} then entails showing that $\sum_{t=1}^{T}Q(\W_t;\B_t)$ is close to $\sum_{t=1}^{T}Q_t(\W_t)$, which itself is bounded by the average instantaneous loss of the iterates.

We now present the main ideas in proving Lemma~\ref{lem:loss_ub}, which closely follows~\cite{ji2019polylogarithmic}. Since $L^{(t)}_t(\U) \leq e^{-y_t \langle \nabla g_t(\W_t), \U \rangle}$, the proof entails lower bounding $y_t \langle \nabla g_t(\W_t), \U \rangle$, which can be decomposed as follows
\begin{align}
y_t \langle \nabla g_t(\W_t), \U \rangle &= y_t \langle \nabla g_t(\W_1), \W_1 \rangle + y_t \langle \nabla g_t(\W_t) - \nabla g_t(\W_1), \W_1 \rangle \nonumber \\
&\qquad \qquad + \lambda y_t \langle \nabla g_t(\W_1), \V \rangle + \lambda y_t \langle \nabla g_t(\W_t) - \nabla g_i(\B_t\W_1), \V \rangle. \label{eq:U}
\end{align}
By homogeneity of the ReLU activations, the first term in Equation~\eqref{eq:U} precisely computes $y_t g_t(\W_1)$, which cannot be too negative under the initialization scheme used in Algorithm~\ref{alg:dropout}, as we show in Lemma~\ref{lem:small_init}.
On the other hand, we show in Lemma~\ref{lem:good_init} that under Assumption~\ref{ass:margin}, $\V$ has a good margin with respect to the randomly initialized weights $\W_1$, so that the third term in Equation~\eqref{eq:U} is concentrated around the margin parameter $\gamma$.
The second and the fourth terms in Equation~\eqref{eq:U} can be bounded thanks to the lazy regime, where $\W_t$ remains close to $\W_1$ at all times. In particular, provided $\|\w_{r,t}-\w_{r,1}\|\leq \frac{7\lambda}{2\gamma \sqrt{m}}$, we show in Lemma~\ref{lem:lazy} that at most only $O(1/\sqrt{m})$-fraction of neural activations change, and thus $\nabla g_t(\W_t) - \nabla g_t(\W_1)$ has a small norm. Lemma~\ref{lem:loss_ub} then follows from carefully choosing $\lambda$ and $m$ such that the right hand side of Equation~\eqref{eq:U} is sufficiently positive.

\section{Experimental Results}\label{sec:exp}
\begin{figure*}[t]
\centering
\vspace*{-15pt}
\begin{tabular}{ccc}
\hspace*{-23pt} 
\includegraphics[width=0.344\textwidth]{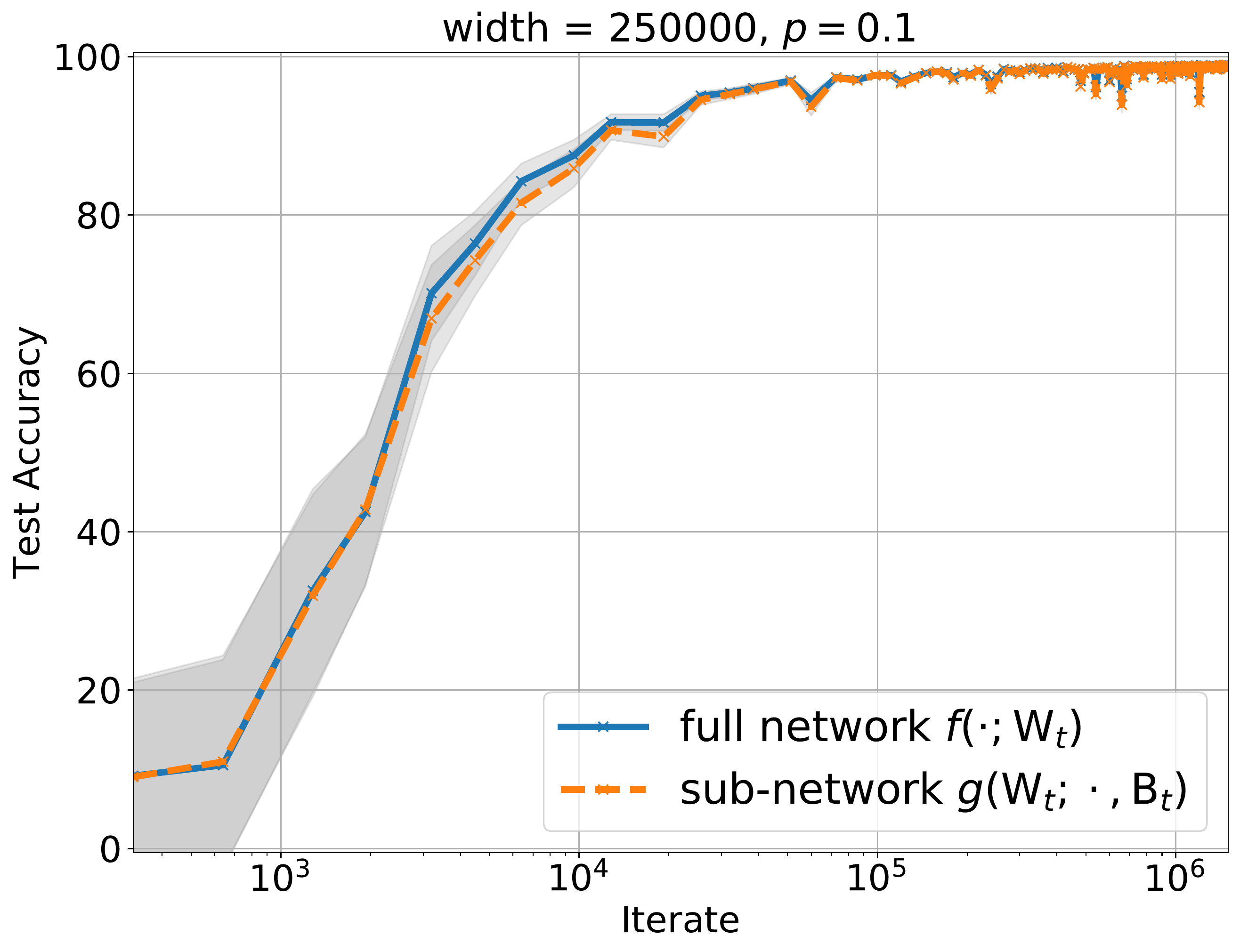}
&
\hspace*{-12pt} 
\includegraphics[width=0.344\textwidth]{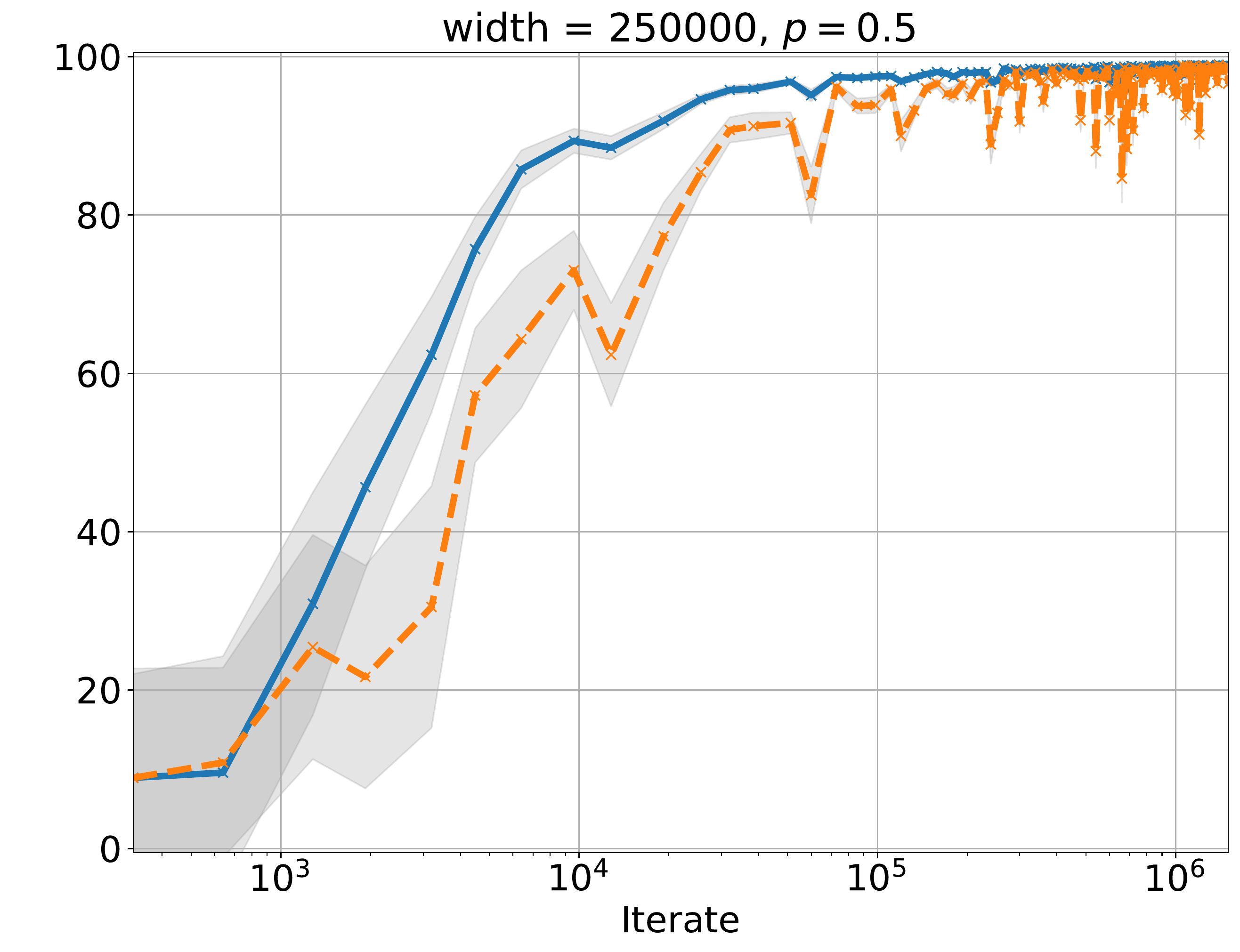}
&
\hspace*{-12pt} 
\includegraphics[width=0.344\textwidth]{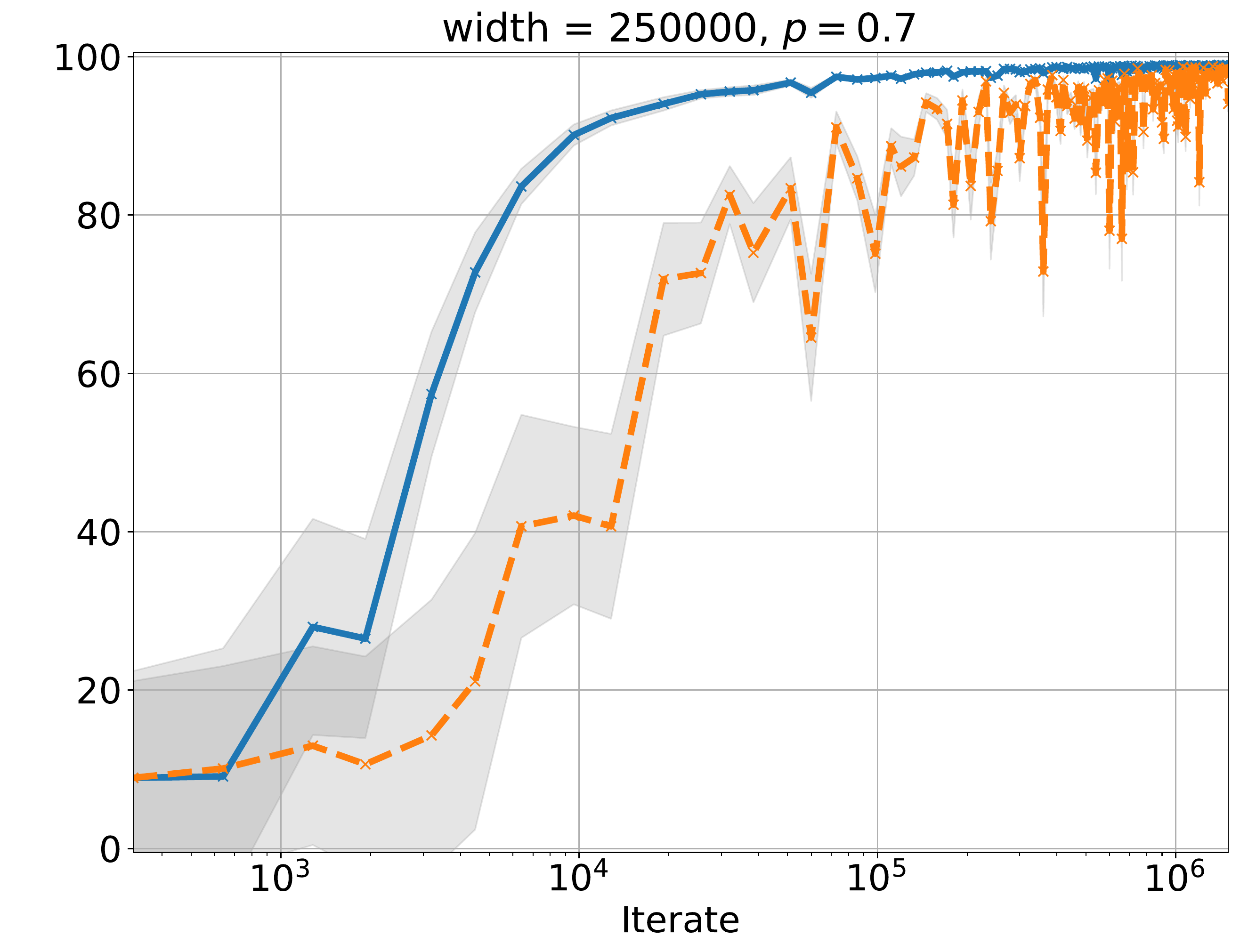}
\\
\hspace*{-23pt} 
\includegraphics[width=0.344\textwidth]{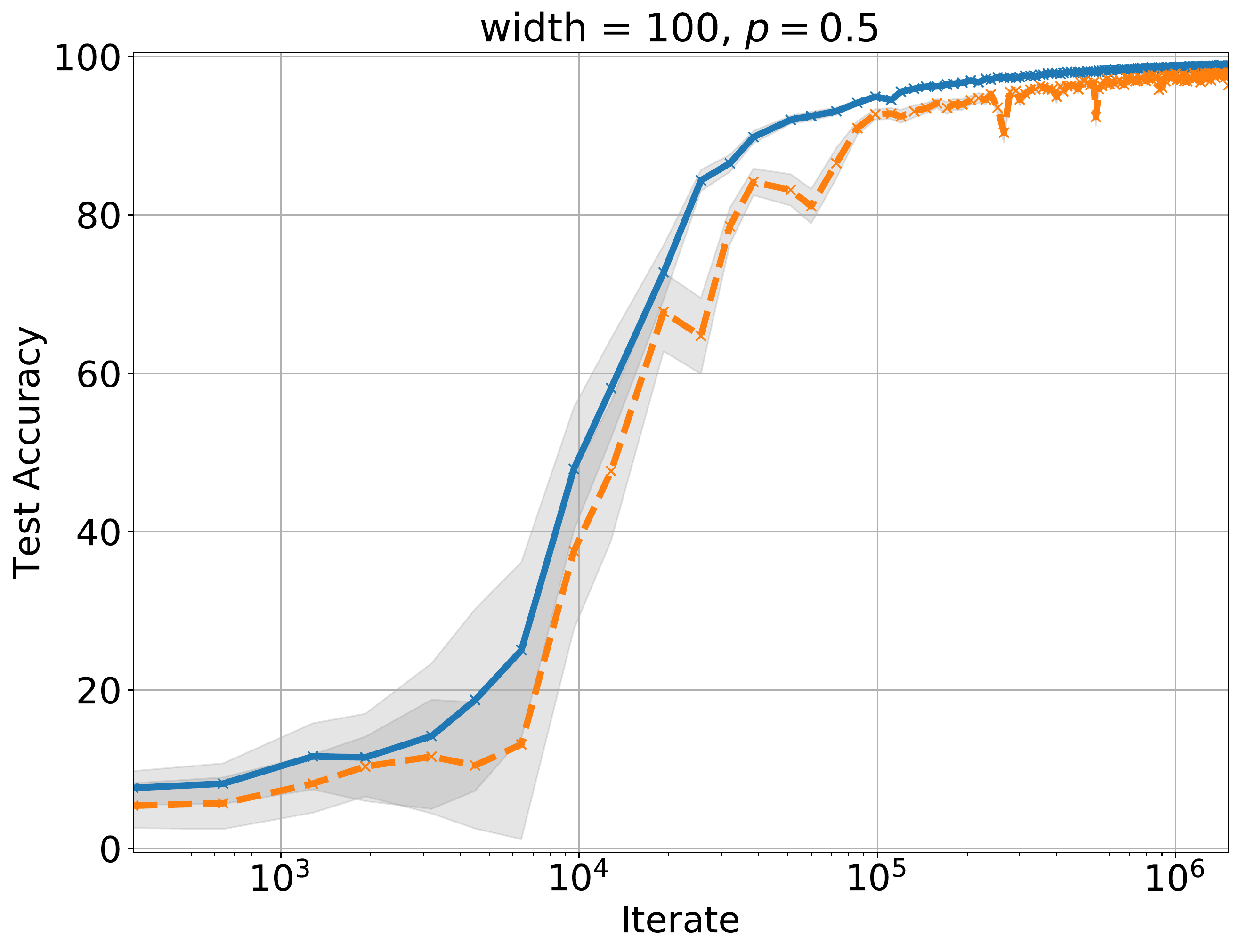}
&
\hspace*{-12pt} 
\includegraphics[width=0.344\textwidth]{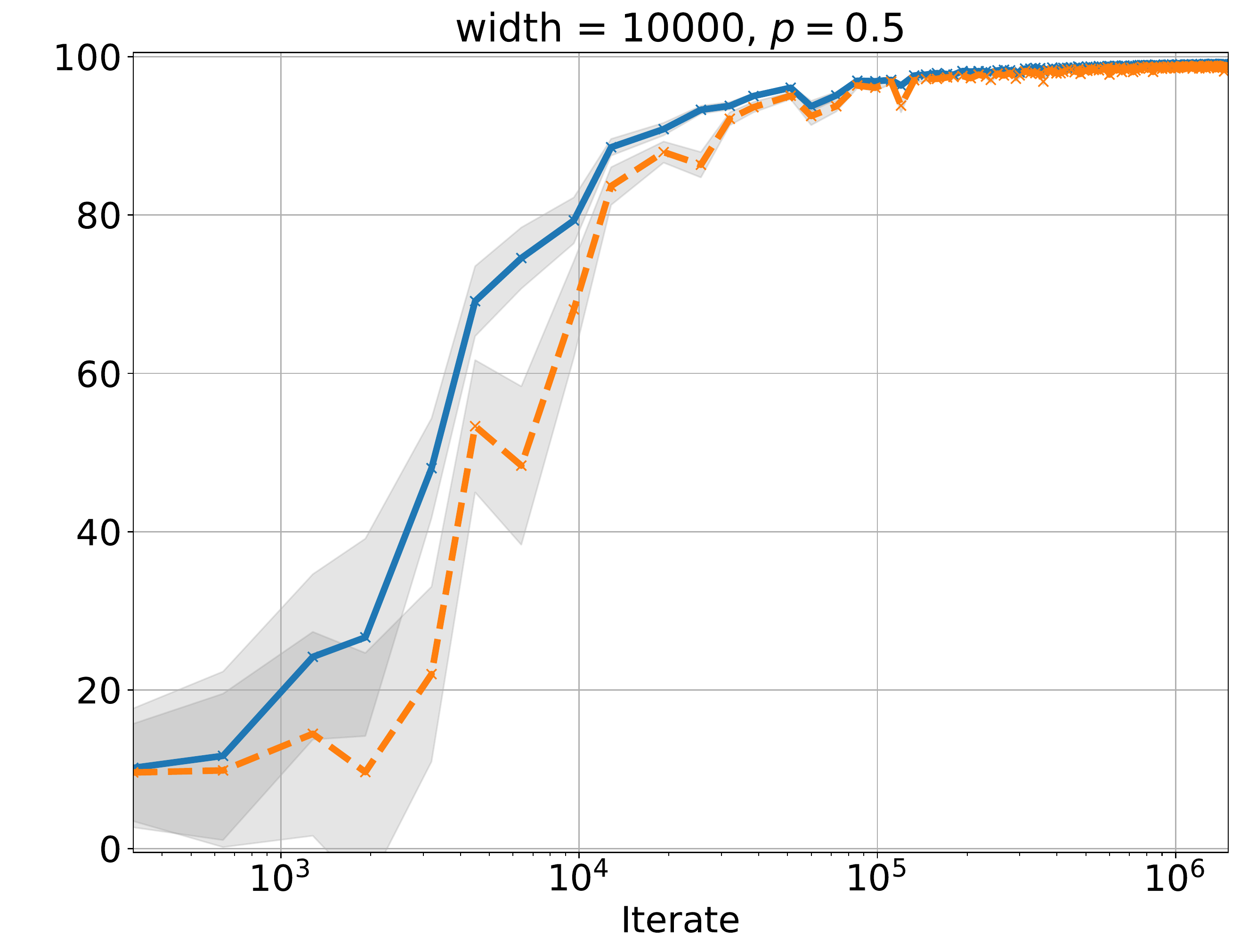}
&
\hspace*{-12pt} 
\includegraphics[width=0.344\textwidth]{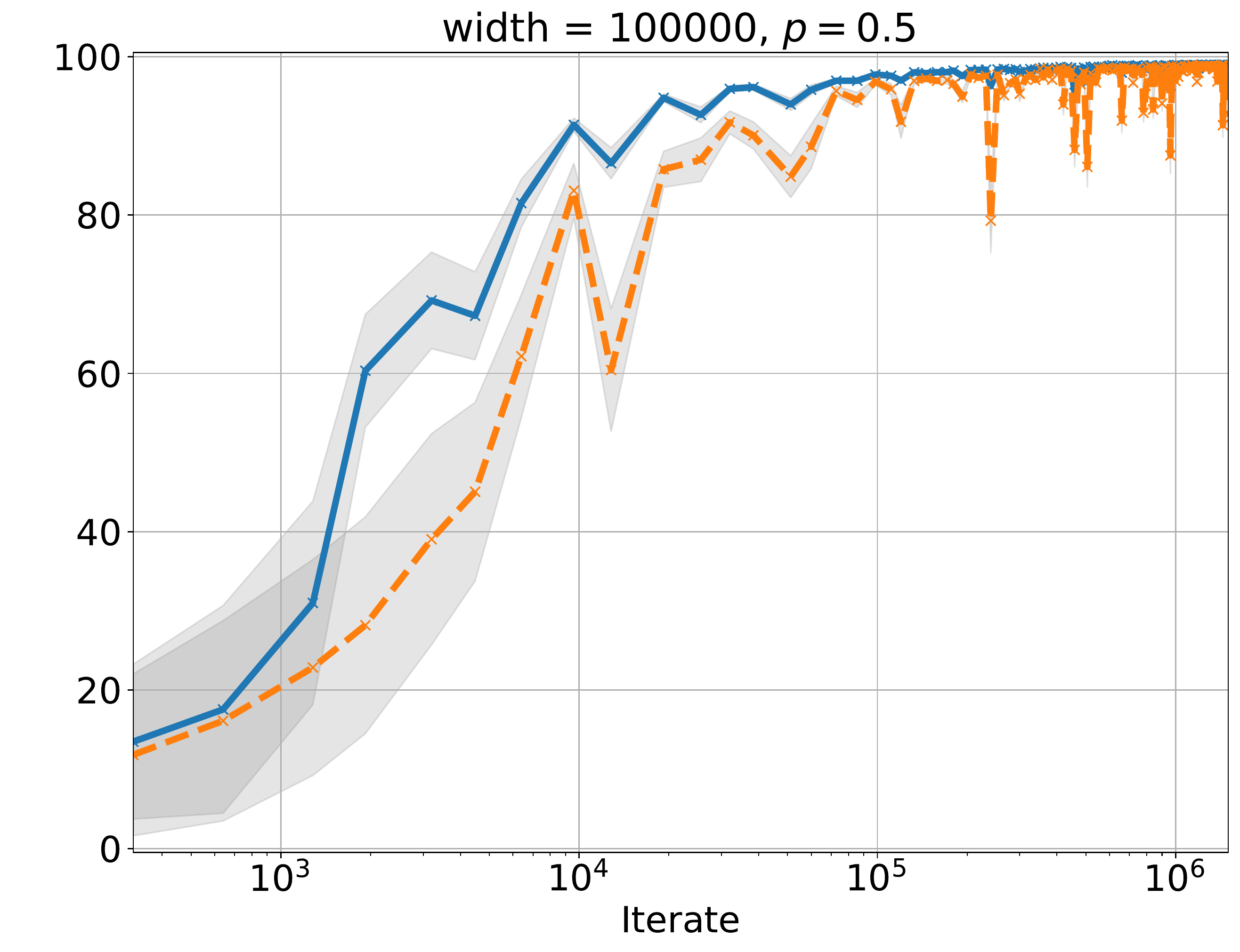}
\vspace*{-5pt} 
\end{tabular}
\caption{\label{fig:full_visited}Test accuracy of the full network $f(\cdot;\W_t)$ as well as the sub-networks $g(\W_t;\cdot, \B_t)$ drawn by dropout iterates, as a function of number of iterations $t$, for ({\bf top}) fixed width $m=250K$ and several dropout rates $1-p\in\{0.1,0.5,0.7\}$; and ({\bf bottom}) fixed dropout rate $1-p=0.5$ and several widths $m\in\{100, 10K, 100K\}$.}
\end{figure*}

The goal of this section is to investigate if dropout indeed compresses the model, as predicted by Theorem~\ref{thm:sub-network}. In particular, we seek to understand if the (sparse) dropout sub-networks $g(\W; \cdot, \B)$ -- regardless of being visited by dropout during the training -- are comparable to the full network $f(\cdot; \W)$, in terms of the test accuracy.

We train a convolutional neural network with a dropout layer on the top hidden layer, using cross-entropy loss, on the MNIST dataset. The network consists of two convolutional layers with max-pooling, followed by three fully-connected layers. All the activations are ReLU. We use a constant learning rate $\eta = 0.01$ and batch-size equal to $64$ for all the experiments. We train several networks where except for the top layer widths ($m \in \{ 100, 500, 1K, 5K, 10K, 50K, 100K, 250K \}$), all other architectural parameters are fixed. We track the test accuracy over 25 epochs as a function of number of iterations, for the full network, the sub-networks visited by dropout, as well as random but fixed sub-networks that are drawn independently, using the same dropout rate. We run the experiments for several values of the dropout rate, $1-p\in \{ 0.1,0.2,0.3,\ldots, 0.9 \}$.

Figure~\ref{fig:full_visited} shows the test accuracy of the full network $f(\cdot;\W_t)$ (blue, solid curve) as well as the dropout iterates $g(\W_t;\cdot, \B_t)$ (orange, dashed curve), as a function of the number of iterations. Both curves are averaged over 50 independent runs of the experiment; the grey region captures the standard deviation. It can be seen that the (sparse) sub-networks drawn by dropout during the training, are indeed comparable to the full network in terms of the generalization error. As expected, the gap between the full network and the sparse sub-networks is higher for narrower networks, and for higher dropout rates. This figure verifies our compression result in Theorem~\ref{thm:sub-network}.

Next, we show that dropout also generalizes to sub-networks that were not observed during the training. In other words, random sub-networks drawn from the same Bernoulli distribution, also performed well. We run the following experiment on the same convolutional network architecture described above with  widths $m\in\{ 100, 1K, 10K, 100K \}$. We draw 100 sub-networks $g(\W;\cdot,\B_1), \ldots, g(\W;\cdot,\B_{100})$, corresponding to diagonal Bernoulli matrices $\B_1,\ldots,\B_{100}$, all generated by the same Bernoulli distribution used at training (Bernoulli parameter $p=0.2$, i.e., dropout rate $1-p=0.8$). In Figure~\ref{fig:full_independent}, we plot the generalization error of these sub-networks as well as the full network as a function of iteration number, as orange and blue curves, respectively. We observe that, as the width increases, the sub-networks become increasingly more competitive; it is remarkable that the effective width of these \emph{typical} sub-networks are only $\approx 1/5$ of the full network.

\begin{figure*}
\centering
\begin{tabular}{cccc}
\hspace*{-21pt} 
\includegraphics[width=0.261\textwidth]{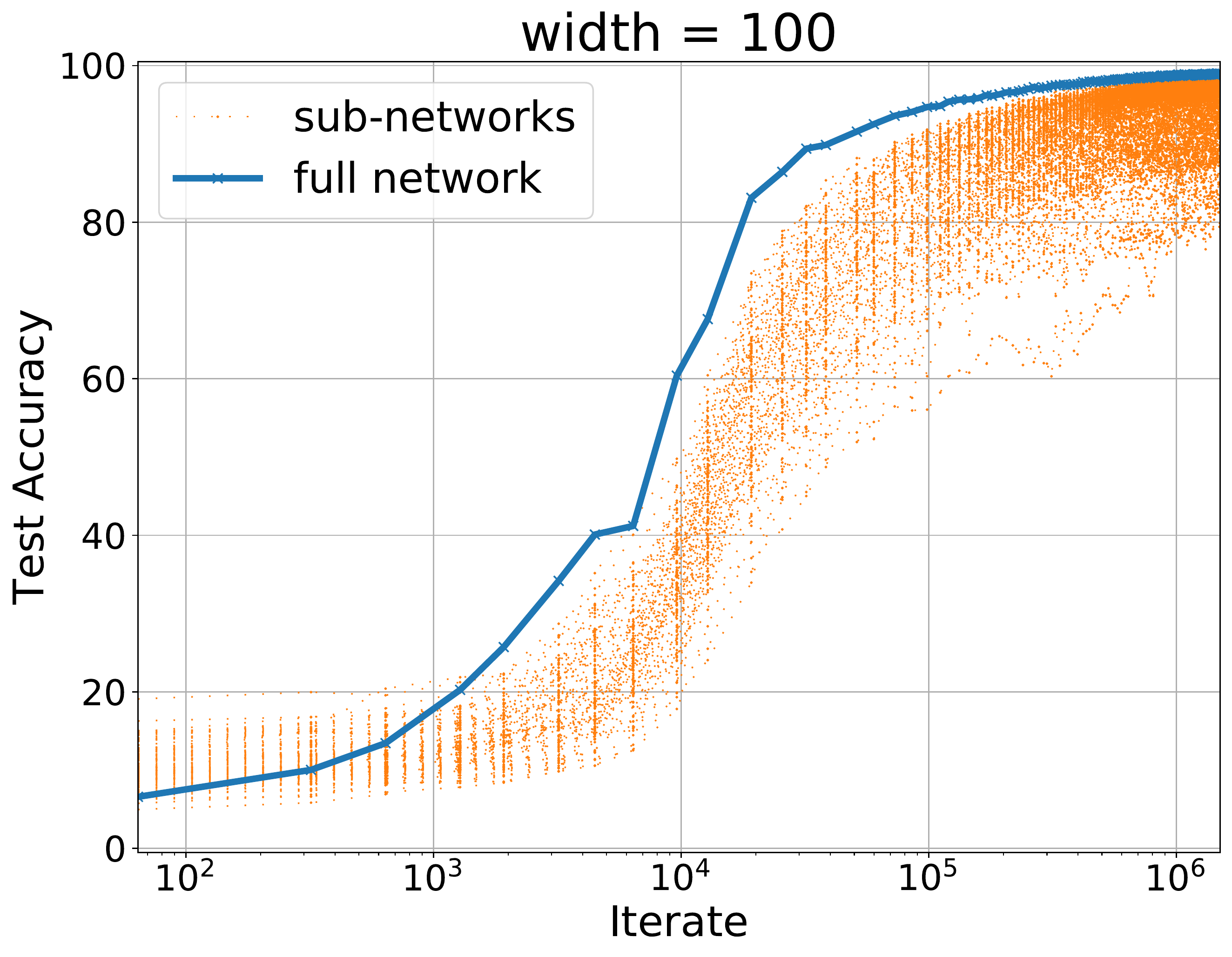}
&
\hspace*{-15pt} 
\includegraphics[width=0.261\textwidth]{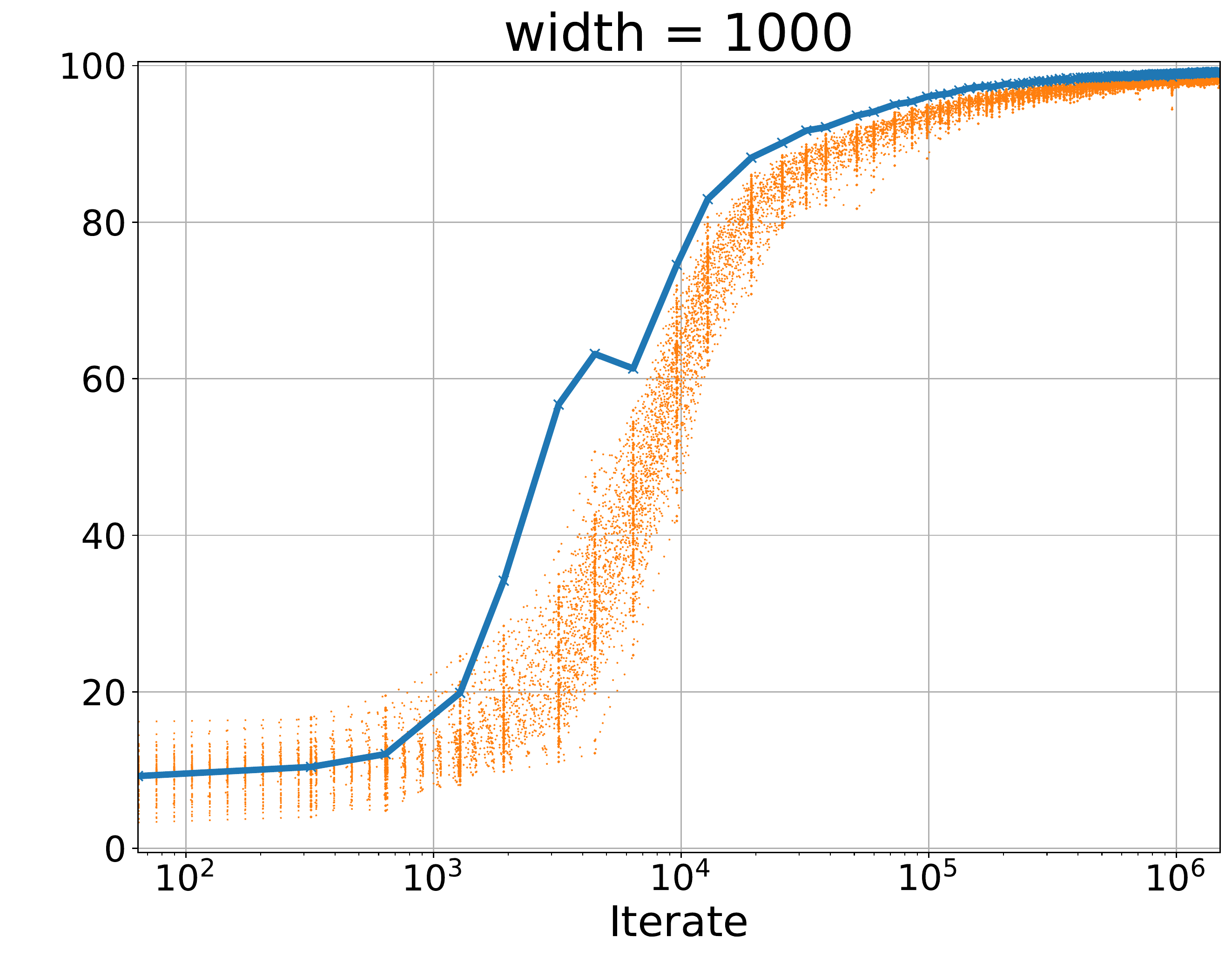}
&
\hspace*{-15pt} 
\includegraphics[width=0.261\textwidth]{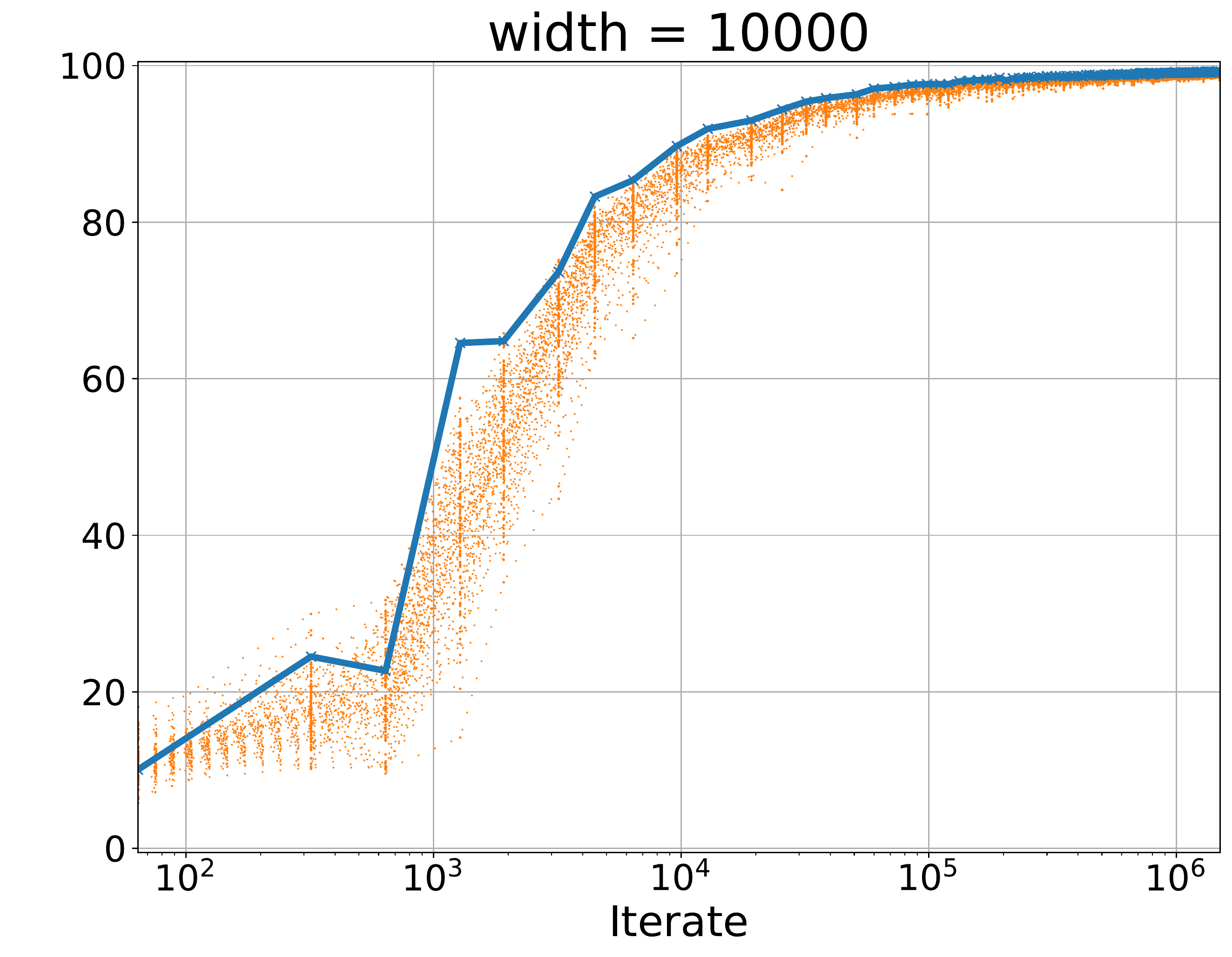}
&
\hspace*{-15pt} 
\includegraphics[width=0.261\textwidth]{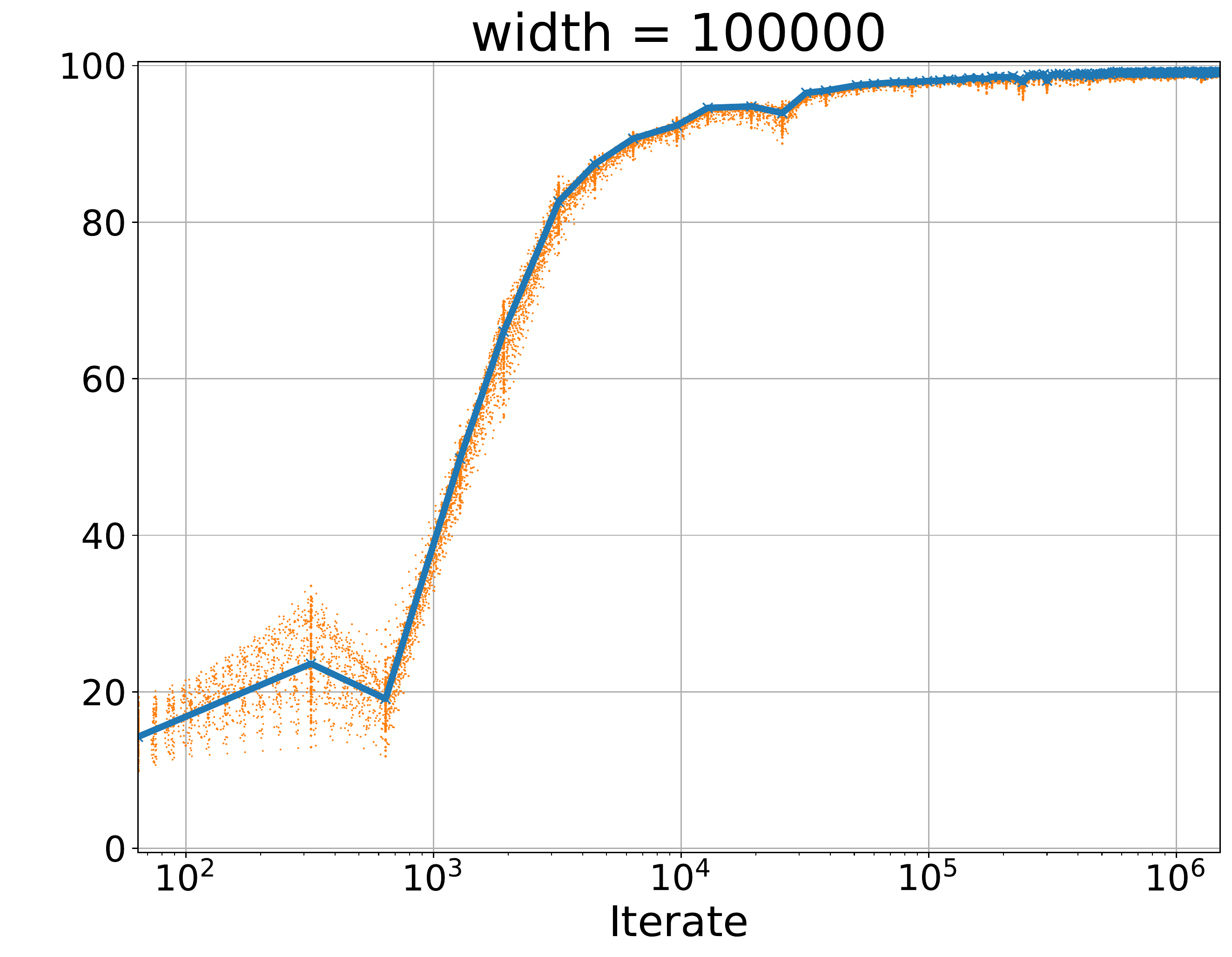}
\end{tabular}
\caption{\label{fig:full_independent}Test accuracy of the full network $f(\cdot;\W_t)$ as well as 100 random sub-networks $g(\W_t; \cdot, \B_1), \ldots, g(\W_t; \cdot, \B_{100})$ with dropout rate $1-p=0.8$, as a function of number of iterations $t$, for several width $m \in \{100,1K,10K,100K \}$.}
\vspace*{-15pt}
\end{figure*}

\section{Conclusion}\label{sec:disc}
Most of the results in the literature of over-parameterized neural networks focus on GD-based methods without any explicit regularization. On the other hand, recent theoretical investigations have challenged the view that \emph{implicit bias} due to such GD-based methods can explain  generalization in deep learning~\citep{dauber2020can}. Therefore, it seems crucial to explore algorithmic regularization techniques in over-parameterized neural networks. This paper takes a step towards understanding a popular algorithmic regularization technique in deep learning. In particular, assuming that the data distribution is separable in the RKHS induced by the neural tangent kernel, this paper presents precise iteration complexity results for dropout training in two-layer ReLU networks using the logistic loss.

\section*{Broader Impact}
We investigate the convergence and generalization of a popular algorithmic regularization technique in deep learning. Although we can not think of any direct social impacts per se, we hope such theoretical studies serve the community in long-term, by improving our understanding of the foundations, which shall eventually lead to more powerful machine learning systems.

\section*{Acknowledgements}
This research was supported, in part, by NSF BIGDATA award IIS-1546482, NSF CAREER award IIS-1943251 and NSF TRIPODS award CCF-1934979. Poorya Mianjy acknowledges support as a MINDS fellow. Raman Arora acknowledges support from the Simons Institute as part of the program on the Foundations of Deep Learning and the Institute for Advanced Study (IAS), Princeton, New Jersey, as part of the special year on Optimization, Statistics, and Theoretical Machine Learning.

\bibliographystyle{plainnat}
\bibliography{references}

\begin{thebibliography}{46}
\providecommand{\natexlab}[1]{#1}
\providecommand{\url}[1]{\texttt{#1}}
\expandafter\ifx\csname urlstyle\endcsname\relax
  \providecommand{\doi}[1]{doi: #1}\else
  \providecommand{\doi}{doi: \begingroup \urlstyle{rm}\Url}\fi

\bibitem[Allen-Zhu et~al.(2018)Allen-Zhu, Li, and Song]{allen2018convergence}
Zeyuan Allen-Zhu, Yuanzhi Li, and Zhao Song.
\newblock A convergence theory for deep learning via over-parameterization.
\newblock \emph{arXiv preprint arXiv:1811.03962}, 2018.

\bibitem[Allen-Zhu et~al.(2019)Allen-Zhu, Li, and Liang]{allen2019learning}
Zeyuan Allen-Zhu, Yuanzhi Li, and Yingyu Liang.
\newblock Learning and generalization in overparameterized neural networks,
  going beyond two layers.
\newblock In \emph{Advances in neural information processing systems}, pages
  6155--6166, 2019.

\bibitem[Arora et~al.(2020)Arora, Bartlett, Mianjy, and
  Srebro]{arora2020dropout}
Raman Arora, Peter Bartlett, Poorya Mianjy, and Nathan Srebro.
\newblock Dropout: Explicit forms and capacity control.
\newblock \emph{arXiv preprint arXiv:2003.03397}, 2020.

\bibitem[Arora et~al.(2019)Arora, Du, Hu, Li, and Wang]{arora2019fine}
Sanjeev Arora, Simon~S Du, Wei Hu, Zhiyuan Li, and Ruosong Wang.
\newblock Fine-grained analysis of optimization and generalization for
  overparameterized two-layer neural networks.
\newblock \emph{arXiv preprint arXiv:1901.08584}, 2019.

\bibitem[Baldi and Sadowski(2013)]{baldi2013understanding}
Pierre Baldi and Peter~J Sadowski.
\newblock Understanding dropout.
\newblock In \emph{Advances in Neural Information Processing Systems (NIPS)},
  pages 2814--2822, 2013.

\bibitem[Beygelzimer et~al.(2011)Beygelzimer, Langford, Li, Reyzin, and
  Schapire]{beygelzimer2011contextual}
Alina Beygelzimer, John Langford, Lihong Li, Lev Reyzin, and Robert Schapire.
\newblock Contextual bandit algorithms with supervised learning guarantees.
\newblock In \emph{Proceedings of the Fourteenth International Conference on
  Artificial Intelligence and Statistics}, pages 19--26, 2011.

\bibitem[Brutzkus et~al.(2017)Brutzkus, Globerson, Malach, and
  Shalev-Shwartz]{brutzkus2017sgd}
Alon Brutzkus, Amir Globerson, Eran Malach, and Shai Shalev-Shwartz.
\newblock {SGD} learns over-parameterized networks that provably generalize on
  linearly separable data.
\newblock \emph{arXiv preprint arXiv:1710.10174}, 2017.

\bibitem[Cao and Gu(2019)]{cao2019generalization}
Yuan Cao and Quanquan Gu.
\newblock Generalization bounds of stochastic gradient descent for wide and
  deep neural networks.
\newblock In \emph{Advances in Neural Information Processing Systems}, pages
  10835--10845, 2019.

\bibitem[Chizat et~al.(2018)Chizat, Oyallon, and Bach]{chizat2018lazy}
Lenaic Chizat, Edouard Oyallon, and Francis Bach.
\newblock On lazy training in differentiable programming. arxiv e-prints, page.
\newblock \emph{arXiv preprint arXiv:1812.07956}, 2018.

\bibitem[Dahl et~al.(2013)Dahl, Sainath, and Hinton]{dahl2013improving}
George~E Dahl, Tara~N Sainath, and Geoffrey~E Hinton.
\newblock Improving deep neural networks for {LVCSR} using rectified linear
  units and dropout.
\newblock In \emph{2013 IEEE international conference on acoustics, speech and
  signal processing}, pages 8609--8613. IEEE, 2013.

\bibitem[Daniely(2017)]{daniely2017sgd}
Amit Daniely.
\newblock Sgd learns the conjugate kernel class of the network.
\newblock In \emph{Advances in Neural Information Processing Systems}, pages
  2422--2430, 2017.

\bibitem[Dauber et~al.(2020)Dauber, Feder, Koren, and Livni]{dauber2020can}
Assaf Dauber, Meir Feder, Tomer Koren, and Roi Livni.
\newblock Can implicit bias explain generalization? stochastic convex
  optimization as a case study.
\newblock \emph{arXiv preprint arXiv:2003.06152}, 2020.

\bibitem[Du et~al.(2018)Du, Lee, Li, Wang, and Zhai]{du2018gradient1}
Simon~S Du, Jason~D Lee, Haochuan Li, Liwei Wang, and Xiyu Zhai.
\newblock Gradient descent finds global minima of deep neural networks.
\newblock \emph{arXiv preprint arXiv:1811.03804}, 2018.

\bibitem[Du et~al.(2019)Du, Zhai, Poczos, and Singh]{du2018gradient}
Simon~S. Du, Xiyu Zhai, Barnabas Poczos, and Aarti Singh.
\newblock Gradient descent provably optimizes over-parameterized neural
  networks.
\newblock In \emph{International Conference on Learning Representations}, 2019.
\newblock URL \url{https://openreview.net/forum?id=S1eK3i09YQ}.

\bibitem[Frankle and Carbin(2019)]{frankle2018the}
Jonathan Frankle and Michael Carbin.
\newblock The lottery ticket hypothesis: Finding sparse, trainable neural
  networks.
\newblock In \emph{International Conference on Learning Representations}, 2019.
\newblock URL \url{https://openreview.net/forum?id=rJl-b3RcF7}.

\bibitem[Gal and Ghahramani(2016)]{gal2016dropout}
Yarin Gal and Zoubin Ghahramani.
\newblock Dropout as a bayesian approximation: Representing model uncertainty
  in deep learning.
\newblock In \emph{Int. Conf. Machine Learning (ICML)}, 2016.

\bibitem[Gao and Zhou(2016)]{gao2016dropout}
Wei Gao and Zhi-Hua Zhou.
\newblock Dropout rademacher complexity of deep neural networks.
\newblock \emph{Science China Information Sciences}, 59\penalty0 (7):\penalty0
  072104, 2016.

\bibitem[Gomez et~al.(2019)Gomez, Zhang, Swersky, Gal, and
  Hinton]{gomez2019learning}
Aidan~N Gomez, Ivan Zhang, Kevin Swersky, Yarin Gal, and Geoffrey~E Hinton.
\newblock Learning sparse networks using targeted dropout.
\newblock \emph{arXiv preprint arXiv:1905.13678}, 2019.

\bibitem[Helmbold and Long(2015)]{helmbold2015inductive}
David~P Helmbold and Philip~M Long.
\newblock On the inductive bias of dropout.
\newblock \emph{Journal of Machine Learning Research (JMLR)}, 16:\penalty0
  3403--3454, 2015.

\bibitem[Hinton et~al.(2012)Hinton, Srivastava, Krizhevsky, Sutskever, and
  Salakhutdinov]{hinton2012improving}
Geoffrey~E Hinton, Nitish Srivastava, Alex Krizhevsky, Ilya Sutskever, and
  Ruslan~R Salakhutdinov.
\newblock Improving neural networks by preventing co-adaptation of feature
  detectors.
\newblock \emph{arXiv preprint arXiv:1207.0580}, 2012.

\bibitem[Hu et~al.(2020)Hu, Li, and Yu]{Hu2020Simple}
Wei Hu, Zhiyuan Li, and Dingli Yu.
\newblock Simple and effective regularization methods for training on noisily
  labeled data with generalization guarantee.
\newblock In \emph{International Conference on Learning Representations}, 2020.
\newblock URL \url{https://openreview.net/forum?id=Hke3gyHYwH}.

\bibitem[Jacot et~al.(2018)Jacot, Gabriel, and Hongler]{jacot2018neural}
Arthur Jacot, Franck Gabriel, and Cl{\'e}ment Hongler.
\newblock Neural tangent kernel: Convergence and generalization in neural
  networks.
\newblock In \emph{Advances in neural information processing systems}, pages
  8571--8580, 2018.

\bibitem[Ji and Telgarsky(2019)]{ji2019polylogarithmic}
Ziwei Ji and Matus Telgarsky.
\newblock Polylogarithmic width suffices for gradient descent to achieve
  arbitrarily small test error with shallow relu networks.
\newblock \emph{arXiv preprint arXiv:1909.12292}, 2019.

\bibitem[Lee et~al.(2019)Lee, Xiao, Schoenholz, Bahri, Novak, Sohl-Dickstein,
  and Pennington]{lee2019wide}
Jaehoon Lee, Lechao Xiao, Samuel Schoenholz, Yasaman Bahri, Roman Novak, Jascha
  Sohl-Dickstein, and Jeffrey Pennington.
\newblock Wide neural networks of any depth evolve as linear models under
  gradient descent.
\newblock In \emph{Advances in neural information processing systems}, pages
  8570--8581, 2019.

\bibitem[Li et~al.(2019)Li, Soltanolkotabi, and Oymak]{li2019gradient}
Mingchen Li, Mahdi Soltanolkotabi, and Samet Oymak.
\newblock Gradient descent with early stopping is provably robust to label
  noise for overparameterized neural networks.
\newblock \emph{arXiv preprint arXiv:1903.11680}, 2019.

\bibitem[Li and Liang(2018)]{li2018learning}
Yuanzhi Li and Yingyu Liang.
\newblock Learning overparameterized neural networks via stochastic gradient
  descent on structured data.
\newblock In \emph{Advances in Neural Information Processing Systems}, pages
  8157--8166, 2018.

\bibitem[McAllester(2013)]{mcallester2013pac}
David McAllester.
\newblock A {PAC}-bayesian tutorial with a dropout bound.
\newblock \emph{arXiv preprint arXiv:1307.2118}, 2013.

\bibitem[Merity et~al.(2017)Merity, Keskar, and Socher]{merity2017regularizing}
Stephen Merity, Nitish~Shirish Keskar, and Richard Socher.
\newblock Regularizing and optimizing {LSTM} language models.
\newblock \emph{arXiv preprint arXiv:1708.02182}, 2017.

\bibitem[Mianjy and Arora(2019)]{mianjy2019dropout}
Poorya Mianjy and Raman Arora.
\newblock On dropout and nuclear norm regularization.
\newblock In \emph{International Conference on Machine Learning}, pages
  4575--4584, 2019.

\bibitem[Mianjy et~al.(2018)Mianjy, Arora, and Vidal]{mianjy2018implicit}
Poorya Mianjy, Raman Arora, and Rene Vidal.
\newblock On the implicit bias of dropout.
\newblock In \emph{International Conference on Machine Learning}, pages
  3540--3548, 2018.

\bibitem[Molchanov et~al.(2017)Molchanov, Ashukha, and
  Vetrov]{molchanov2017variational}
Dmitry Molchanov, Arsenii Ashukha, and Dmitry Vetrov.
\newblock Variational dropout sparsifies deep neural networks.
\newblock In \emph{Proceedings of the 34th International Conference on Machine
  Learning-Volume 70}, pages 2498--2507. JMLR. org, 2017.

\bibitem[Neyshabur et~al.(2017)Neyshabur, Tomioka, Salakhutdinov, and
  Srebro]{neyshabur2017geometry}
Behnam Neyshabur, Ryota Tomioka, Ruslan Salakhutdinov, and Nathan Srebro.
\newblock Geometry of optimization and implicit regularization in deep
  learning.
\newblock \emph{arXiv preprint arXiv:1705.03071}, 2017.

\bibitem[Nitanda and Suzuki(2019)]{nitanda2019refined}
Atsushi Nitanda and Taiji Suzuki.
\newblock Refined generalization analysis of gradient descent for
  over-parameterized two-layer neural networks with smooth activations on
  classification problems.
\newblock \emph{arXiv preprint arXiv:1905.09870}, 2019.

\bibitem[Oymak and Soltanolkotabi(2020)]{oymak2020towards}
Samet Oymak and Mahdi Soltanolkotabi.
\newblock Towards moderate overparameterization: global convergence guarantees
  for training shallow neural networks.
\newblock \emph{IEEE Journal on Selected Areas in Information Theory}, 2020.

\bibitem[Senen-Cerda and Sanders(2020)]{senen2020almost}
Albert Senen-Cerda and Jaron Sanders.
\newblock Almost sure convergence of dropout algorithms for neural networks.
\newblock \emph{arXiv preprint arXiv:2002.02247}, 2020.

\bibitem[Song and Yang(2019)]{song2019quadratic}
Zhao Song and Xin Yang.
\newblock Quadratic suffices for over-parametrization via matrix chernoff
  bound.
\newblock \emph{arXiv preprint arXiv:1906.03593}, 2019.

\bibitem[Srivastava et~al.(2014)Srivastava, Hinton, Krizhevsky, Sutskever, and
  Salakhutdinov]{srivastava2014dropout}
Nitish Srivastava, Geoffrey~E Hinton, Alex Krizhevsky, Ilya Sutskever, and
  Ruslan Salakhutdinov.
\newblock Dropout: a simple way to prevent neural networks from overfitting.
\newblock \emph{Journal of Machine Learning Research (JMLR)}, 15\penalty0 (1),
  2014.

\bibitem[Szegedy et~al.(2015)Szegedy, Liu, Jia, Sermanet, Reed, Anguelov,
  Erhan, Vanhoucke, and Rabinovich]{szegedy2015going}
Christian Szegedy, Wei Liu, Yangqing Jia, Pierre Sermanet, Scott Reed, Dragomir
  Anguelov, Dumitru Erhan, Vincent Vanhoucke, and Andrew Rabinovich.
\newblock Going deeper with convolutions.
\newblock In \emph{Proceedings of the IEEE conference on computer vision and
  pattern recognition}, pages 1--9, 2015.

\bibitem[Vershynin(2018)]{vershynin2018high}
Roman Vershynin.
\newblock \emph{High-dimensional probability: An introduction with applications
  in data science}, volume~47.
\newblock Cambridge University Press, 2018.

\bibitem[Wager et~al.(2013)Wager, Wang, and Liang]{wager2013dropout}
Stefan Wager, Sida Wang, and Percy~S Liang.
\newblock Dropout training as adaptive regularization.
\newblock In \emph{Advances in Neural Information Processing Systems (NIPS)},
  2013.

\bibitem[Wager et~al.(2014)Wager, Fithian, Wang, and Liang]{wager2014altitude}
Stefan Wager, William Fithian, Sida Wang, and Percy~S Liang.
\newblock Altitude training: Strong bounds for single-layer dropout.
\newblock In \emph{Adv. Neural Information Processing Systems}, 2014.

\bibitem[Wan et~al.(2013)Wan, Zeiler, Zhang, Le~Cun, and
  Fergus]{wan2013regularization}
Li~Wan, Matthew Zeiler, Sixin Zhang, Yann Le~Cun, and Rob Fergus.
\newblock Regularization of neural networks using dropconnect.
\newblock In \emph{International conference on machine learning}, pages
  1058--1066, 2013.

\bibitem[Wei et~al.(2019)Wei, Lee, Liu, and Ma]{wei2019regularization}
Colin Wei, Jason~D Lee, Qiang Liu, and Tengyu Ma.
\newblock Regularization matters: Generalization and optimization of neural
  nets vs their induced kernel.
\newblock In \emph{Advances in Neural Information Processing Systems}, pages
  9709--9721, 2019.

\bibitem[Wei et~al.(2020)Wei, Kakade, and Ma]{wei2020implicit}
Colin Wei, Sham Kakade, and Tengyu Ma.
\newblock The implicit and explicit regularization effects of dropout.
\newblock \emph{arXiv preprint arXiv:2002.12915}, 2020.

\bibitem[Zhai and Wang(2018)]{zhai2018adaptive}
Ke~Zhai and Huan Wang.
\newblock Adaptive dropout with rademacher complexity regularization.
\newblock In \emph{International Conference on Learning Representations}, 2018.
\newblock URL \url{https://openreview.net/forum?id=S1uxsye0Z}.

\bibitem[Zou et~al.(2018)Zou, Cao, Zhou, and Gu]{zou2018stochastic}
Difan Zou, Yuan Cao, Dongruo Zhou, and Quanquan Gu.
\newblock Stochastic gradient descent optimizes over-parameterized deep relu
  networks.
\newblock \emph{arXiv preprint arXiv:1811.08888}, 2018.

\end{thebibliography}

\clearpage
\appendix
\begin{center}
{\bf\Large ``On Convergence and Generalization of Dropout Training''}
\end{center}
\vspace{30pt}

\section{Auxiliary Theorems}
\begin{theorem}[Gaussian Concentration~\cite{vershynin2018high}]\label{thm:gaussian_concentration}
Consider a random vector $\z \sim \cN(\0,\I_d)$ and a $\rho$-Lipschitz function $f:\R^{d}\to \R$ (with respect to the Euclidean metric). Then $f(\z)$ is $\rho$-sub-Gaussian and it holds for all $t \geq 0$:
\begin{equation*}
\prob\{{f(\z) - \E[f(\z)]} \geq t\} \leq e^{\frac{-t^2}{2\rho^2}}
\end{equation*}
\end{theorem}

\begin{theorem}[Hoeffding's inequality~\cite{vershynin2018high}]\label{thm:hoeffding}
Let $X_1, \ldots , X_n$ be independent, mean zero random variables. Assume that $X_i \in [m_i,M_i]$ for every $i$. Then, for every $t >  0$, we have
\begin{align*}
\prob\{ \sum_{i=1}^{n}X_i \geq t \} \leq e^{ -\frac{2 t^2}{\sum_{i=1}^{n}(m_i - M_i)^2}}
\end{align*}
\end{theorem}

\begin{theorem}[Theorem~1 of \cite{beygelzimer2011contextual}]\label{thm:martingale}
Let $X_1, \ldots, X_T$ be a sequence of real-valued random variables. Let $\E_t[Y]:=\E[Y|X_1,\ldots,X_{t-1}]$. Assume, for all $t$, that $X_t \leq R$ and that $\E_t[X_t] = 0$. Define the random variable $S_t:=\sum_{k=1}^{t}X_t$, and $V_t:=\sum_{k=1}^{t}\E_k[X_k^2]$. Then for any $\delta > 0$, with probability at least $1 - \delta$, we have the following guarantee:
\begin{equation*}
S_t \leq R\ln{\frac{1}{\delta}} + (e-2)\frac{V_t}{R}
\end{equation*}
\end{theorem}

\section{Proofs}
The following Lemma bounds $\abs{R_t}$, where $R_t:=\{r\in[m] | \  \bI\{ \w_{r,t}^\top\x_t >0 \} \neq \bI\{\w_{r,1}^\top\x_t > 0\}  \}$ is the set of hidden nodes at time $t$ whose activation on sample $\x_t$ is different from the initialization.
\begin{lemma}\label{lem:lazy}
Assume that $\| \w_{r,1} - \w_{r,t}\| \leq D$ holds for all $r\in [m]$, where $D$ is a positive constant. Then, with probability at least $1-\frac{\delta}{3}$, we have that
\begin{equation*}
| R_t  | \leq mD + \sqrt\frac{m\ln{3T/\delta}}{2}, \text{ for all } t \in [T].
\end{equation*}
\end{lemma}

\begin{proof}[Proof of Lemma~\ref{lem:lazy}]
Assume that $r \in R_t$. Then it holds that
\begin{align*}
| \w_{r,1}^\top \x_t | &\leq | \w_{r,1}^\top \x_t | + | \w_{r,t}^\top \x_t | \\
&= | (\w_{r,1} - \w_{r,t})^\top \x_t  | \tag{$r \in R_t$}\\
&\leq \|\w_{r,1} - \w_{r,t}\| \|\x_t\| \tag{Cauchy-Schwarz} \\
&= \| \w_{r,1} - \w_{r,t}\| \leq D \tag{$\|\x_t\|=1$}
\end{align*}
Since $\w_{r,1}^\top\x_t$ is a standard Gaussian random variable, by anti-concentration property of the Gaussian distribution, $\E[\bI\{\abs{\w_{r,1}^\top\x_t}\leq D \}] = \Pr\{\abs{\w_{r,1}^\top\x_t}\leq D \} \leq \frac{2D}{\sqrt{2\pi}}$.
On the other hand, we have that
\begin{align*}
|R_t| = \abs{\{r | \ \bI\{ \w_{r,t}^\top\x_t >0 \} \neq \bI\{\w_{r,1}^\top\x_t > 0\}  \}} &\leq | \{r | \ \abs{\w_{r,1}^\top\x_t}\leq D  \} | = \sum_{r=1}^{m}\bI\{\abs{\w_{r,1}^\top\x_t}\leq D \}
\end{align*}
By Hoeffding's inequality, we have the following with probability at least $1-\frac{\delta}{3T}$:
\begin{align*}
\frac1m \sum_{r=1}^{m}\bI\{\abs{\w_{r,1}^\top\x_t}\leq D \} \leq \Pr\{\abs{\w_{r,1}^\top\x_t}\leq D \} + \sqrt\frac{\ln{3 T/\delta}}{2m} \leq \frac{2D}{\sqrt{2\pi}} + \sqrt\frac{\ln{3 T/\delta}}{2m}.
\end{align*}
Multiplying both sides by $m$ and applying union bound on $t\in[T]$ completes the proof.
\end{proof}

\begin{lemma}\label{lem:expected}
For any $t\in[T]$, let $\cB_{t}:=\{\B_1,\ldots,\B_{t}\}$ denote the set of Bernoulli masks up to time $t$. Then it holds almost surely that:
\begin{equation}\label{eq:jensens}
\sum_{t=1}^{T} \ell(y_t f_t(q\W_t))\leq \E_{\cB_T}[\sum_{t=1}^{T} L_t(\W_t)].
\end{equation}
\end{lemma}
\begin{proof}[Proof of Lemma~\ref{lem:expected}]
For any $a,b\in\R$, the function $\ell(z)=\log(1+\exp(az+b))$ is convex in $z$. We have the following inequalities: 
\begin{align*}
\E_{\cB_T}[\sum_{t=1}^{T} L_t(\W_t)] &= \sum_{t=1}^{T}\E_{\cB_t}[\ell(y_t \cdot \frac{1}{\sqrt{m}} \a^\top \B_t\sigma(\W_t\x_t))] \\
&= \sum_{t=1}^{T}\E_{\cB_{t-1}}[\E_{\B_t}\ell(y_t \cdot \frac{1}{\sqrt{m}} \sum_{r=1}^{m} a_r b_{r,t} \sigma(\w_{r,t}^\top \x_t)) | \cB_{t-1}] \tag{smoothing property} \\
&\geq \sum_{t=1}^{T}\E_{\cB_{t-1}}[\ell(y_t \cdot \frac{1}{\sqrt{m}} \sum_{r=1}^{m} a_r \E_{\B_t}[b_{r,t}] \sigma(\w_{r,t}^\top \x_t)) | \cB_{t-1}] \tag{Jensen's inequality} \\
&= \sum_{t=1}^{T}\ell(y_t \cdot \frac{1}{\sqrt{m}} \sum_{r=1}^{m} a_r \sigma(q\w_{r,t}^\top \x_t)) \tag{$\E[b_{r,t}] = q$, homogeneity of ReLU} \\
&= \sum_{t=1}^{T}\ell(y_t f_t(q\W_t))
\end{align*}
which completes the proof.
\end{proof}

\begin{proof}[Proof of Lemma~\ref{lem:lyapunov}]
Using the dropout update rule in Algorithm~\ref{alg:dropout}, we start by analyzing the distance of consecutive iterates from the reference point $\U$, assuming that $\Pi_c(\U)=\U$:
\begin{align*}
\|\W_{t+1}-\U\|_F^2 &= \|\Pi_c(\W_{t+\frac12})-\U\|_F^2\\
&\leq \|\W_{t+\frac12}-\U\|_F^2 \tag{$\U \in \cW_c$}\\
&= \|\W_{t} - \eta \nabla L_t(\W_t) -\U\|_F^2\\
&= \|\W_{t} - \U \|_F^2 -2 \eta \langle \nabla L_t(\W_t), \W_t -\U\rangle + \eta^2 \|\nabla L_t(\W_t)\|_F^2
\end{align*}
The last term on the right hand side above is bounded as follows:
\begin{align*}
\eta^2 \|\nabla L_t(\W_t)\|_F^2 &= \eta^2 \| \ell'(y_t g_t(\W_t)) y_t \nabla g_t(\W_t)\|_F^2 \\
&= \eta^2 \left(-\ell'(y_t g_t(\W_t))\|\nabla g_t(\W_t)\|_F \right)^2 \\
&= \eta^2 Q_t(\W_t)^2 \sum_{r=1}^{m} \| \frac{\partial g_t(\W_t)}{\partial \w_{r,t}}\|^2   \\
&\leq \eta^2  Q_t(\W_t)^2 \tag{$\| \frac{\partial g_t(\W_t)}{\partial \w_{r,r}}\|\leq \frac{1}{\sqrt{m}}$}\\
&\leq \frac{\eta^2}{\ln{2}}  Q_t(\W_t) \tag{$ Q_t(\cdot) \leq 1/\ln{2}$}\\
&\leq \eta  Q_t(\W_t) \tag{assumption $\eta \leq \ln{2}$}\\
&\leq \eta L_t(\W_t) \tag{$ Q_t(\cdot)\leq L_t(\cdot)$}
\end{align*}
The second term can be bounded as follows:
\begin{align*}
\langle \nabla L_t(\W_t), \W_t -\U \rangle &= \ell'(y_t g_t(\W_t)) \langle y_t \nabla g_t(\W_t) , \W_t-\U \rangle\\
&= \ell'(y_t g_t(\W_t))(y_t g_t(\W_t) - y_t g_t^{(t)}(\U)) \tag{Homogeneity, definition of $g_t^{(t)}$} \\
&\geq (\ell(y_t g_t(\W_t)) -\ell(y_t g_t^{(t)}(\U))) \tag{convexity of $\ell(\cdot)$} \\
&= L_t(\W_t) - L^{(t)}_t(\U)
\end{align*}
Plugging back the above inequalities we get
\begin{align}
\|\W_{t+1}-\U\|_F^2 \leq \|\W_{t+\frac{1}{2}}-\U\|_F^2 &\leq \|\W_{t} - \U \|_F^2 -2 \eta (L_t(\W_t) - L^{(t)}_t(\U)) + \eta L_t(\W_t) \nonumber \\
&= \|\W_{t} - \U \|_F^2 - \eta L_t(\W_t) +2\eta L^{(t)}_t(\U) \label{eq:distance_ub}
\end{align}
Rearranging, dividing both sides by $\eta$, and averaging over iterates we arrive at
\begin{align*}
\frac1T\sum_{t=1}^{T}L_t(\W_t) &\leq \sum_{t=1}^{T}\frac{\|\W_{t} - \U \|_F^2 - \|\W_{t+1}-\U\|_F^2}{\eta T} + \frac2T\sum_{t=1}^{T} L^{(t)}_t(\U) \\
&\leq \frac{\|\W_1 - \U \|_F^2}{\eta T} + \frac2T\sum_{t=1}^{T}  L^{(t)}_t(\U) \tag{Telescopic sum}
\end{align*}
\end{proof}

\begin{lemma}\label{lem:small_init}
With probability at least $1-\delta/3$ it holds uniformly over all $t\in[T]$ that $|g_t(\W_1)|\leq \sqrt{2\ln{6T/\delta}}$, provided that $m \geq 25 \ln{6T/\delta}$.
\end{lemma}

\begin{proof}[Proof of Lemma~\ref{lem:small_init}] The proof is similar to the proof of Lemma~{A.1} in~\cite{ji2019polylogarithmic}, except for that we have to take into account the randomness due to dropout as well. In particular, there are four different sources of randomness in $g_t(\W_1) = g(\W_1;\x_t,\B_t)$: 1) the randomly initialized hidden layer weights $\W_1$, 2) the randomly initialized top layer weights $\a$, 3) the input vector $\x_t, \ t \in [T]$, and 4) the Bernoulli masks $\B_t, \ t \in [T]$. Given input $\x_t$ and the dropout mask $\B_t$, let $\h_{t}(\W) = \frac{1}{\sqrt{m}}\B_t\sigma(\W\x_t)\in\R^m$ denote the (scaled) output of the dropout layer with hidden weights $\W$. It is easy to see that the function $g:\W \mapsto \|\h_{t}(\W)\|$ is $1$-Lipschitz:
\begin{align*}
\abs{g(\W) - g(\W')} &=
\abs{\|\h_{t}(\W)\| - \|\h_{t}(\W')\|} \\
&\leq \|\h_{t}(\W) - \h_{t}(\W')\| \tag{Reverse Triangle Inequality}\\
&= \sqrt{\sum_{r=1}^{m}(\frac{1}{\sqrt{m}}b_i^{(t)}\sigma(\langle\w_{r,1},\x_t\rangle) - \frac{1}{\sqrt{m}}b_i^{(t)}\sigma(\langle\w'_{r,1},\x_t\rangle))^2}\\
&= \frac{\sqrt{\sum_{r=1}^{m}(\langle\w_{r,1},\x_t\rangle - \langle\w'_{r,1},\x_t\rangle)^2}}{\sqrt{m}} \tag{1-Lipschitzness of ReLU}\\
&\leq \frac{\sqrt{\sum_{r=1}^{m}\|\w_{r,1}-\w'_{r,1}\|^2 \|\x_t\|^2}}{\sqrt{m}} \tag{Cauchy-Schwarz}\\
&= \frac{\|\W-\W'\|_F}{\sqrt{m}}
\end{align*}
Using Gaussian concentration (Lemma~\ref{thm:gaussian_concentration}), we get that $\|\h_{t}(\W_1)\| - \E_{\W_1}[\|\h_{t}(\W_1)\|] \leq \sqrt\frac{{2\ln{\frac{6T}{\delta}}}}{{m}}$ with probability at least $1-\frac{\delta}{6T}$. It also holds that:
\begin{align*}
\E_{\W_1}[\|\h_{t}(\W_1)\|] &\leq \sqrt{\E_{\W_1}[\|\h_{t}(\W_1)\|^2]}\\
&= \sqrt{\sum_{r=1}^{m}\E_{\w_{r,1}}(\frac{1}{\sqrt{m}}b_{r,t}\sigma(\w_{r,1}^\top\x_t))^2}\\
&\leq \sqrt\frac{{\sum_{r=1}^{m}\E_{\w_{r,1}}[\sigma(\w_{r,1}^\top\x_t)^2]}}{{m}}\\
&= \sqrt{\E_{z\sim\cN(0,1)}[\sigma(z)^2]} = \frac{1}{\sqrt{2}}
\end{align*}
As a result, we have with probability at least $1-\frac{\delta}{6T}$ that $\|\h_{t}(\W_1)\| \leq \sqrt\frac{2\ln{6T/\delta}}{m} + \frac{\sqrt{2}}{2} \leq 1$ whenever $m \geq 25 \ln{6T/\delta}$. Now, taking a union bound over all $t\in [T]$, we get that $\|\h_{t}(\W_1)\| \leq 1$ holds simultaneously for all iterates. Conditioned on this event, the random variable $g_t(\W_1) = \langle \a, \h_{t}(\W_1) \rangle $ is zero mean and $1$-sub-Gaussian, so that by the general Hoeffding's inequality, for any $t$, with probability at least $1-\frac{\delta}{6T}$, it holds that $\abs{g_t(\W_1)}\leq \sqrt{2\ln{6T/\delta}}$. Taking union bound over all $t\in[T]$, with probability $1-\delta/6$ it holds that $\abs{g_t(\W_1)}\leq \sqrt{2\ln{6T/\delta}}$ simultaneously for all $t\in[T]$. Finally, the probability that both of these events hold is no less than $(1-\delta/6)^2\geq 1-\delta/3 $, which completes the proof.
\end{proof}

\begin{lemma}\label{lem:good_init}
Under Assumption~\ref{ass:margin}, for any $\delta\in(0,1)$, with probability at least $1-\delta/3$ it holds uniformly for all $t \in [T]$ that:
\begin{equation*}
y_t g_t^{(1)}(\V) = y_t \langle \nabla g_t(\W_1), \V\rangle \geq \gamma - \sqrt{\frac{2\ln{3T/\delta}}{m}}
\end{equation*}
\end{lemma}

\begin{proof}[Proof of Lemma~\ref{lem:good_init}]
By Assumption~\ref{ass:margin}, it holds that $\E_{\z, b} [y \langle \psi(\z), b \x \bI\{\z^\top\x > 0\}\rangle ] \geq \gamma$ for all $(\x,y)$ in the domain of $\cD$. We observe that $y_t g_t^{(1)}(\V)$ is an empirical estimate of this quantity:
\begin{align*}
y_t g_t^{(1)}(\V) &= y_t \langle \nabla g_t(\W_1), \V\rangle \\
&= y_t \sum_{r=1}^{m}\langle \frac{1}{\sqrt{m}}a_r b_{r,t}\bI\{\x_t^\top \w_{r,1} > 0\}\x_t, \frac{1}{\sqrt{m}}a_r \psi(\w_{r,1})\rangle \\
&= \frac{1}{m}\sum_{r=1}^{m}y_t \langle \psi(\w_{r,1}), b_{r,t} \x_t \bI\{\w_{r,1}^\top\x_t > 0\} \rangle
\end{align*}
For $t,r\in [T] \times [m]$, let $\gamma_{t,r} := y_t \langle \psi(\w_{r,1}), b_{r,t} \x_t \bI\{\w_{r,1}^\top\x_t > 0\} \rangle$. Note that $\E_{\W_1,\B_t}[\gamma_{t,r}] =  \E_{\z, b}[y_t \langle \psi(\z), \b \x_t \bI\{\z^\top\x_t > 0\} \rangle]$. Also, for any $t$, the random variable $\gamma_{t,r}$ is bounded almost surely as follows:
\begin{align*}
\abs{\gamma_{t,r}} 
&\leq \abs{y_t}  \|\psi(\w_{r,1})\| \abs{b_{r,t}} \|\x_t\| \abs{\bI\{\w_{r,1}^\top\x_t > 0\}}  \leq 1.
\end{align*}
Therefore by Hoeffding’s inequality (Theorem~\ref{thm:hoeffding}), with probability at least $1 - \frac{\delta}{3T}$, it holds that:
\begin{equation*}
y_t g_t^{(1)}(\V) - \gamma \geq y_t g_t^{(1)}(\V) - \E[y_t g_t^{(1)}(\V)] \geq - \sqrt{\frac{2\ln{3T/\delta}}{m}}
\end{equation*}
Applying a union bound over $t$ finishes the proof.
\end{proof}

\begin{proof}[Proof of Lemma~\ref{lem:loss_ub}]
We adopt the proof of Theorem~2.2 in~\cite{ji2019polylogarithmic} for dropout training. Assume that $\|\w_{r,t}-\w_{r,1}\|\leq \frac{7 \lambda}{2\gamma\sqrt{m}}$ holds for the first $T$ iterates of Algorithm~\ref{alg:dropout}. Then with probability at least $1-(\frac{\delta}{3}+\frac{\delta}{3}+\frac{\delta}{3}) = 1-\delta$, Lemma~\ref{lem:lazy}, Lemma~\ref{lem:small_init}, and Lemma~\ref{lem:good_init} hold simultaneously. We first prove that $L^{(t)}_t(\U) \leq \frac{\lambda^2}{2\eta T}$ for all $t\in[T]$. Using the inequality $\log(1+z)\leq z$, we get that 
\begin{equation*}
L^{(t)}_t(\U)
= \log(1+e^{-y_t \langle \nabla g_t(\W_t), \U \rangle}) \leq e^{-y_t \langle \nabla g_t(\W_t), \U \rangle}
\end{equation*}
To upper-bound the right hand side, we lower-bound $y_t \langle \nabla g_t(\W_t), \U \rangle$. By definition of $\U$, we have
\begin{equation}\label{eq:1st}
y_t \langle \nabla g_t(\W_t), \U \rangle = y_t \langle \nabla g_t(\W_t), \W_1 \rangle + \lambda y_t \langle \nabla g_t(\W_t), \V \rangle
\end{equation}
We bound each of the terms separately. The first term can be decomposed as follows:
\begin{align}
y_t \langle \nabla g_t(\W_t), \W_1 \rangle &= y_t \langle \nabla g_t(\W_1), \W_1 \rangle + y_t \langle \nabla g_t(\W_t) - \nabla g_t(\W_1), \W_1 \rangle \nonumber \\
&\geq -\abs{y_t g_t(\W_1)} - \abs{y_t \langle \nabla g_t(\W_t) - \nabla g_t(\W_1), \W_1 \rangle}
\end{align}
By Lemma~\ref{lem:small_init}, the first term on right hand side is lower-bounded by $-\abs{g_t(\W_1)}\geq -\sqrt{2\ln{6T/\delta}}$. We bound the second term as follows:
\begin{align}
| y_t \langle \nabla g_t(\W_t) &- \nabla g_t(\W_1), \W_1 \rangle | = \abs{ \frac{y_t}{\sqrt{m}}\sum_{r=1}^{m}a_rb_{r,t}(\bI\{\w_{r,t}^\top\x_t > 0\}-\bI\{\w_{r,1}^\top\x_t > 0\})\w_{r,1}^\top\x_t } \nonumber\\
&\leq    \frac{1}{\sqrt{m}} \sum_{r\in R_t } \abs{a_r b_{r,t} \langle \w_{r,1}, \x_t \rangle } \tag{Triangle inequality}\\
&\leq   \frac{1}{\sqrt{m}} \sum_{r\in R_t } \abs{ \langle \w_{r,t} - \w_{r,1}, \x_t \rangle } \tag{$r\in R_t $}\\
&\leq   \frac{\abs{R_t } \|\w_{r,t} - \w_{r,1}\|}{\sqrt{m}} \nonumber\\
&\leq   \frac{49 \lambda^2 }{4\gamma^2 \sqrt{m}} + \sqrt\frac{49 \lambda^2\ln{3T/\delta}}{8\gamma^2 m} \tag{Lemma~\ref{lem:lazy}} \nonumber\\
&\leq   \frac{\lambda \gamma }{2} \label{eq:2nd}
\end{align}
where the last inequality holds when $m \geq \max\{98 \gamma^{-4} \ln{3T/\delta}, 2401 \gamma^{-6} \lambda^2\} = 2401 \gamma^{-6} \lambda^2$.
The second term in Equation~\ref{eq:1st} is bounded as follows:
\begin{align} 
y_t \langle \nabla g_t(&\W_t), \V \rangle =   y_t \langle \nabla g_t(\W_1), \V \rangle +   y_t \langle \nabla g_t(\W_t) - \nabla g_t(\W_1), \V \rangle \nonumber\\
&\geq y_t \langle \nabla g_t(\W_1), \V \rangle -   \abs{ y_t \langle \nabla g_t(\W_t) - \nabla g_t(\W_1), \V \rangle} \nonumber\\
&= y_t g_t^{(1)}(\V) -   \abs{ \frac{y_t}{\sqrt{m}}\sum_{r=1}^{m}a_rb_{r,t}(\bI\{\w_{r,t}^\top\x_t > 0\}-\bI\{\w_{r,1}^\top\x_t > 0\})\langle \frac{1}{\sqrt{m}}a_r\psi(\w_{r,1}),\x_t\rangle } \nonumber\\
&\geq \gamma - \sqrt{\frac{2\ln{3T/\delta}}{m}} - \frac{1}{m} \sum_{r\in R_t } \abs{a_r b_{r,t} \langle \psi(\w_{r,1}), \x_t\rangle } \tag{Lemma~\ref{lem:good_init}}  \nonumber\\
&\geq \gamma - \sqrt{\frac{2\ln{3T/\delta}}{m }}  - \frac{  \abs{R_t }}{ m} \nonumber\\
&\geq \gamma - \sqrt{\frac{2\ln{3T/\delta}}{m}}  - \frac{7 \lambda}{2\gamma \sqrt{m}} - \sqrt\frac{\ln{3T/\delta}}{2m} \tag{Lemma~\ref{lem:lazy}} \nonumber\\
&\geq \gamma - \frac{\gamma^2}{7} - \frac{\gamma^2}{14} - \frac{\gamma^2}{14} = \gamma - \frac{2\gamma^2}{7} \geq \frac{5\gamma}{7} \label{eq:3rd}
\end{align}
where the penultimate inequality holds when $m \geq \max\{98 \gamma^{-4} \ln{3T/\delta}, 2401 \gamma^{-6} \lambda^2\} = 2401 \gamma^{-6} \lambda^2$. Plugging Equations~\eqref{eq:2nd} and~\eqref{eq:3rd} in Equation~\ref{eq:1st}, we get that
\begin{equation}\label{eq:margin_lb}
y_t \langle \nabla g_t(\W_t), \U \rangle \geq     - \sqrt{2\ln{6T/\delta}}  + \frac{3\lambda \gamma}{14} \geq \ln{\frac{2 \eta T}{\lambda^2}},
\end{equation}
where the inequality hold true for $\lambda :=  5\gamma^{-1}\ln{2\eta T}+\sqrt{44 \gamma^{-2}\ln{6T/\delta}}$. Thus, we have that 
\begin{align*}
L^{(t)}_t(\U) = \log(1+e^{-y_t \langle \nabla g_t(\W_t), \U \rangle}) &\leq \frac{\lambda^2}{2 \eta T}.
\end{align*}

We now prove by induction that $\|\w_{r,t}-\w_{r,1}\|\leq\frac{7 \lambda}{2\gamma\sqrt{m}}$ holds throughout dropout training. First, we show that the claim holds for $t=2$:
\begin{align*}
\|\w_{r,2} - \w_{r,1}\| = \| \Pi_c(\eta \frac{\partial L_t(\B_1 \W_1)}{\partial \w_{r,1}} ) \| &\leq \| \eta \frac{\partial L_t(\B_1 \W_1)}{\partial \w_{r,1}}  \|  \\
&\leq \|\eta \ell'(y_t f_t(\B_1\W_1))y_i\frac{\partial f_t(\B_1\W_1)}{\partial \w_{r,1}}  \| \\
&\leq \frac{\eta}{\ln{2} \sqrt{m}} \leq \frac{7 \lambda}{2\gamma\sqrt{m}} \tag{$\eta\leq \ln{2}$},
\end{align*}
which proves the basic step. Now by induction hypothesis, we assume that the claim holds for all $k \in [t]$, i.e., it holds that $\|\w_{r,k}-\w_{r,1}\| \leq \frac{7 \lambda}{2\gamma \sqrt{m}}$. Therefore, it holds that $\|\w_{r,k}\|\leq \|\w_{r,1}\| + \|\w_{r,k} - \w_{r,1}\| \leq c - 1 + 1 \leq c$, where we used the triangle inequality, the fact that $\|\w_{r,1}\| \leq c-1$, and that $m\geq 2401\gamma^{-6}\lambda^2$. This, in particular, means that all iterates $1<k\leq t$ remain in $\cW_c$:
\begin{equation}\label{eq:no_projection}
\W_{k} = \Pi_c(\W_{k-\frac{1}{2}}) = \W_{k-\frac12} \text{ for all } 1<k\leq t.
\end{equation}
For the $t+1$-th iterate, we first upper-bound the distance from initialization in terms of the $ Q$ function:
\begin{align*}
\|\w_{r,t+1} - \w_{r,1}\| &= \|\Pi_c(\w_{r,t} - \eta \frac{\partial L_t(\W_t)}{\partial \w_{r,t}}) - \w_{r,1}\| \\
&\leq \|\w_{r,t} - \eta \frac{\partial L_t(\W_t)}{\partial \w_{r,t}} - \w_{r,1}\| \\
&\leq \|\eta \frac{\partial L_t(\W_t)}{\partial \w_{r,t}}\| + \|\w_{r,t} - \w_{r,1}\| \\
&\leq \sum_{k=1}^{t} \|\eta \frac{\partial L_k(\W_k)}{\partial \w_{r,k}}\| \\
&\leq \eta \sum_{k=1}^{t} - \ell'(y_k g_k(\W_k)) \|y_k\frac{\partial g_k(\W_k)}{\partial \w_{r,k}} \| \nonumber\\
&\leq \frac{\eta}{\sqrt{m}} \sum_{k=1}^{t}- \ell'(y_k g_k(\W_k))
\end{align*}
The idea is to turn the right hand side above into a telescopic sum using the identity $\W_{k+1}-\W_k =\W_{k+\frac12}-\W_k = \eta \ell'(y_k g_k(\W_k)) y_k \nabla g_k(\W_k), \ k \in [t-1]$. By induction hypothesis, for all $k\in[t]$, Equation~\eqref{eq:3rd} guarantees $y_k \langle \nabla g_k(\W_k),\V\rangle \geq \frac{5\gamma}{7}$. Thus, multiplying the right hand side of~\eqref{eq:dist1} by $\frac{7}{5\gamma} y_k \langle \nabla g_k(\W_k),\V\rangle$, we get that:
\begin{align}
\|\w_{r,t+1} - \w_{r,1}\| &\leq \frac{7 \eta}{5 \gamma \sqrt{m} } \sum_{k=1}^{t} - \ell'(y_k g_k(\W_k)) y_k \langle \nabla g_k(\W_k),\V\rangle \nonumber\\
&= \frac{7}{5 \gamma \sqrt{m} } \sum_{k=1}^{t}\langle \eta \nabla L_k(\W_k),\V\rangle \nonumber \\
&= \frac{7}{5 \gamma \sqrt{m} } \langle \W_{t+\frac{1}{2}}-\W_1,\V\rangle \tag{Equation~\eqref{eq:no_projection}}\\
&= \frac{7\langle \W_{t+\frac{1}{2}}-\U,\V\rangle + 7\langle \U-\W_1,\V\rangle}{5 \gamma \sqrt{m}} \nonumber\\
&\leq \frac{7\| \W_{t+\frac{1}{2}}-\U\|_F \|\V\|_F + 7\langle \lambda\V,\V\rangle}{5 \gamma \sqrt{m}} \tag{Cauchy-Schwarz} \nonumber\\
&\leq \frac{7\| \W_{t+\frac{1}{2}}-\U\|_F +  7\lambda}{5 \gamma \sqrt{m} } \label{eq:dist1}
\end{align}
Again by induction hypothesis, Equation~\eqref{eq:distance_ub} and Equation~\eqref{eq:margin_lb} hold for all $k\in[t]$, which are used to bound $\|\W_{t+\frac{1}{2}}-\U\|_F$ as follows:
\begin{align}
\|\W_{t+\frac{1}{2}}-\U\|_F^2 &\leq \|\W_{t} - \U \|_F^2 -2 \eta (L_t(\W_t) - L^{(t)}_t(\U)) + \eta L_t(\W_t) \tag{Equation~\eqref{eq:distance_ub}}\\
&\leq \|\W_{t} - \U \|_F^2 +2 \eta L^{(t)}_t(\U)) \nonumber\\
&\leq \|\W_1 - \U\|_F^2 + 2\eta\sum_{k=1}^{t}L^{(k)}_k(\U) \nonumber\\
&\leq \|\lambda\V\|_F^2 + 2 \eta t\frac{\lambda^2}{2 \eta T} \tag{Equation~\ref{eq:margin_lb}} \nonumber\\
&\leq \lambda^2 + \frac{\lambda^2 t}{T} \tag{$\|\V\|_F\leq 1$} \nonumber\\
&\leq 2\lambda^2  \nonumber\\
\implies \|\W_{t+\frac{1}{2}}-\U\|_F &\leq \sqrt{2}\lambda \label{eq:dist2}
\end{align}
Plugging Equation~\eqref{eq:dist2} back in Equation~\eqref{eq:dist1}, we arrive at:
\begin{equation*}
\|\w_{r,t+1}-\w_{r,1}\| \leq \frac{7\sqrt{2}\lambda + 7\lambda}{5 \gamma\sqrt{m}} \leq \frac{7 \lambda}{2\gamma\sqrt{m}}
\end{equation*}
Which completes the induction step and the proof.
\end{proof}

A crucial step in giving generalization bounds in expectation via upper-bounding the logistic
\begin{wrapfigure}{h}{0.4\linewidth}
\centering
\vspace*{-10pt}
\includegraphics[width=0.4\textwidth]{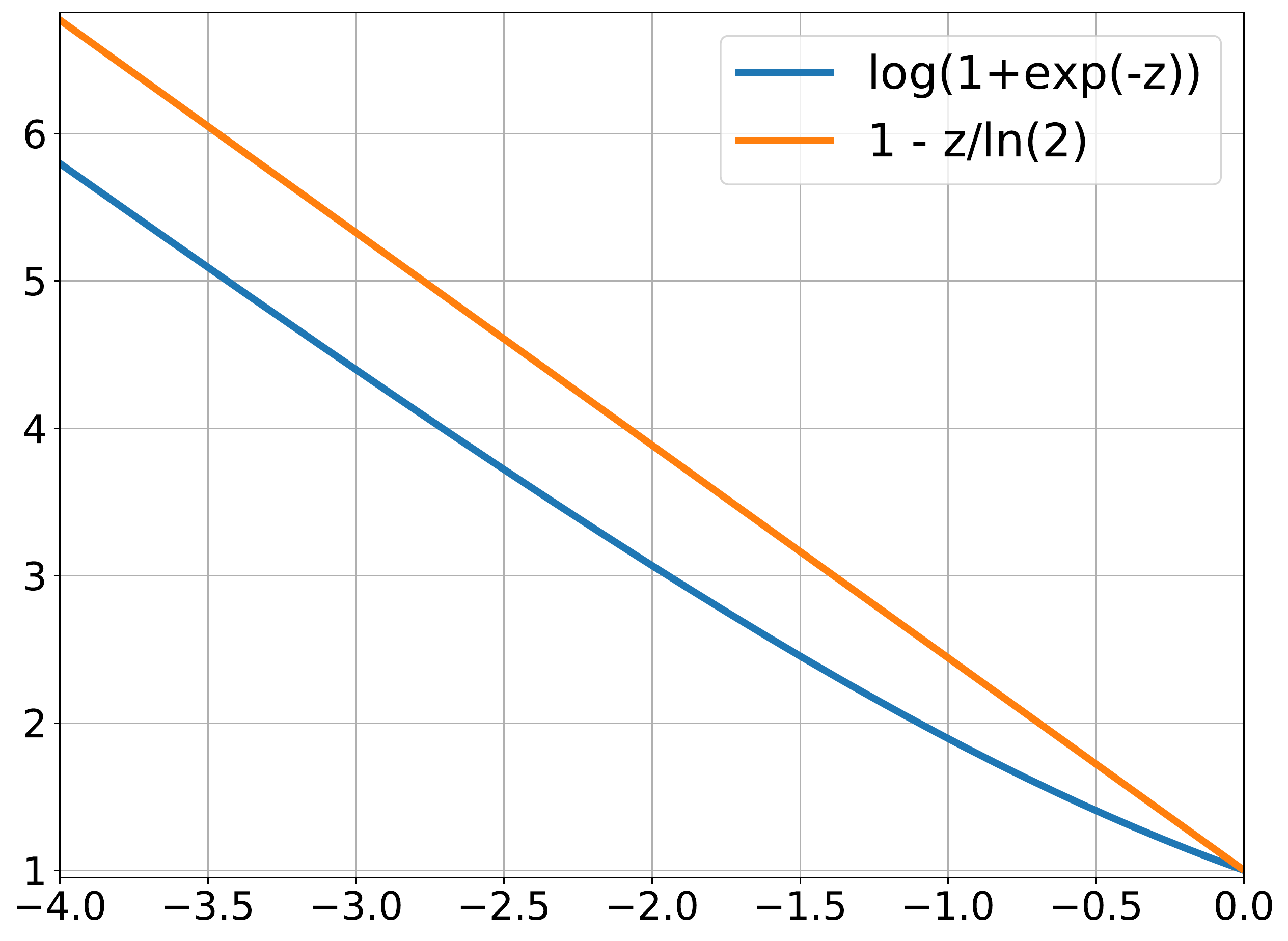}
\label{fig:ub}
\vspace*{-25pt}
\end{wrapfigure}
loss is to control the maximum value the loss can take on any iterate of the algorithm. In particular, we need to upper-bound the instantaneous loss of $g^{(t)}_t(\U)$, which appears in the right hand side of Lemma~\ref{lem:lyapunov}. To that end, we note that the logistic loss only grows linearly for $z < 0$. More formally, it holds for all $z < 0$ that:
\begin{equation}\label{eq:ub_linear}
\log(1+e^{-z}) \leq \frac{-z}{\ln{2}} + 1
\end{equation}
as depicted in Figure~\ref{fig:ub}.

\begin{lemma}\label{lem:worst_case}
Under Algorithm~\ref{alg:dropout}, it holds with probability one for all iterates that $L^{(t)}_t(\U) \leq \frac{c\sqrt{m}}{\ln{2}}+1$.
\end{lemma}
\begin{proof}[Proof of Lemma~\ref{lem:worst_case}]
Recall that $L^{(t)}_t(\U) = \ell(y_t g^{(t)}_t(\U))$. First we bound the argument inside the loss function:
\begin{align*}
\abs{y_t g^{(t)}_t(\W_t)} &= \abs{y_t \langle \nabla g_t(\W_t), \U \rangle}\\
&\leq \sum_{r=1}^{m} \abs{\langle \frac{\partial g_t(\W_t)}{\partial \w_{r,t}}, \u_r \rangle} \tag{triangle inequality}\\
&\leq \sum_{r=1}^{m} \| \frac{\partial g_t(\W_t)}{\partial \w_{r,t}}\| \| \w_{r,1} + \lambda \v_r \| \tag{Cauchy-Schwarz} \\
&\leq \sum_{r=1}^{m} \frac{c-1+\lambda/\sqrt{m}}{\sqrt{m}} \tag{$\|\w_{r,1}\| \leq c-1$, $\|\v_r\|\leq 1/\sqrt{m}$}\\
&\leq c\sqrt{m}\tag{$\lambda \leq \sqrt{m}$}
\end{align*}
Now using Equation~\eqref{eq:ub_linear}, we get that
\begin{align*}
L^{(t)}_t(\U)=  \log(1+e^{-y_t \langle \nabla g_t(\W_t), \U \rangle }) \leq \log(1+\exp(c\sqrt{m})) \leq \frac{c\sqrt{m}}{\ln{2}}+1.
\end{align*}
\end{proof}

\begin{proof}[Proof of Theorem~\ref{thm:overall}] Note that $\W_t$ is conditionally independent from $\x_t$ given $\x_1,\ldots,\x_{t-1}$. Thus, 
\begin{align*}
\E_{\cS_T}[L_t(\W_t)] = \E_{\cS_{t-1}}[\E_{(\x_t,y_t)} \ell(y_t f_t(\W_t))|\cS_{t-1}] = \E_{\cS_{t-1}}[L(\W_t)]
\end{align*}
Using the fact that logistic loss upper-bounds the zero-one loss, taking expectation over initialization, taking average over iterates, and using Lemma~\ref{lem:expected}, we get that:
\begin{align*}
\E_{\W_1,\a,\cS_{T}}[\frac{1}{T}\sum_{t=1}^{T}\cR(q \W_t)] &\leq \E_{\W_1,\a,\cS_T}[\frac{1}{T}\sum_{t=1}^{T}\ell(y_t f_t(q \W_t))] \tag{$\bI\{z<0\} \leq \ell(z)$} \\
&\leq \E_{\W_1,\a,\cS_T,\cB_T}[\frac{1}{T}\sum_{t=1}^{T}L_t(\W_t)] \tag{Lemma~\ref{lem:expected}}\\
&\leq \frac{\E_{\W_1}[\| \W_1 - \U \|_F^2]}{2\eta T} + \frac{2}{T}\sum_{t=1}^{T}\E_{\W_1,\a,\cS_T,\cB_T}[L^{(t)}_t(\U)] \tag{Lemma~\ref{lem:lyapunov}}
\end{align*}
The first term is upper-bounded by $\frac{\lambda^2}{2\eta T}$ since $\| \W_1 - \U \|_F^2 = \| \W_1 - \W_1 - \lambda \V \|_F^2 = \lambda^2 \|\V \|_F^2 \leq \lambda^2$. Bounding the second term is based on the following two facts:
\begin{enumerate}
\item By Lemma~\ref{lem:loss_ub}, with probability at least $1-\delta$, it holds that $L^{(t)}_t(\U) \leq \frac{\lambda^2}{2 \eta T} =: u_1$.
\item By Lemma~\ref{lem:worst_case}, it holds with probability one that $L^{(t)}_t(\U) \leq \frac{c\sqrt{m}}{\ln{2}}+1 \leq 2c\sqrt{m} =: u_2$.
\end{enumerate}
Therefore, the expected value of $L^{(t)}_t(\U)$ can be upper-bounded as:
\begin{align*}
\E[L^{(t)}_t(\U)] \leq (1-\delta) u_1 + \delta u_2 \leq \frac{\lambda^2}{2 \eta T}  + 2 \delta  c \sqrt{m}
\end{align*}
Choosing $\delta := \frac{1}{4 \eta c\sqrt{m} T}$ guarantees that 
\begin{align*}
\E[L^{(t)}_t(\U)] \leq \frac{\lambda^2}{2 \eta T}  + \frac{1}{2 \eta T} \leq \frac{\lambda^2}{\eta T},
\end{align*}
where $\lambda :=  5\gamma^{-1}\ln{2\eta T}+\sqrt{44 \gamma^{-2}\ln{24 \eta c\sqrt{m} T^2}}$.
\end{proof}

\begin{proof}[Proof of Theorem~\ref{thm:sub-network}]
First, recall the following property of the logistic loss:
\begin{equation*}
\bI\{z < 0\} \leq -2\ln{2}\ell'(z) \leq 2\ln{2}\ell(z)
\end{equation*}
which implies that $\cR(\W_t;\B_t) \leq 2\ln{2} Q(\W_t;\B_t)$, where $Q(\W; \B):=\E_\cD[-\ell'(y g(\W;\x,\B)]$ is the expected value of the negative derivative of the logistic loss. On the other hand, taking the empirical average over the training data, and using Lemma~\ref{lem:lyapunov} and Lemma~\ref{lem:loss_ub}, we conclude that the following holds with probability at least $1-\delta$:
\begin{align*}
\frac{1}{T}\sum_{t=1}^{T}Q_t(\W_t) &\leq \frac{1}{T}\sum_{t=1}^{T} L_t(\W_t)\\
&\leq \frac{\|\W_1 - \U \|_F^2}{\eta T} + \frac2T\sum_{t=1}^{T}  L^{(t)}_t(\U)  \tag{Lemma~\ref{lem:lyapunov}} \\
&\leq  \frac{\lambda^2}{\eta T} + \frac{2}{T}\sum_{t=1}^{T}\frac{\lambda^2}{2\eta T} \tag{Lemma~\ref{lem:loss_ub}} \\
&\leq \frac{2\lambda^2}{\eta T}.
\end{align*}
Given the dropout masks $\cB_T$, since $Q(\W_t;\B_t) = \E_\cD[Q_t(\W_t)]$, we know that $\sum_{t=1}^{T} Q(\W_t;\B_t) - \sum_{t=1}^{T}Q_t(\W_t)$ is a martingale difference with respect to the past observations, $\cS_{T-1}$. We next show that the average of $Q_t(\W_t)$ on the right hand side above is close to the average of $Q(\W_t;\B_t)$, using Theorem~\ref{thm:martingale}, similar to Lemma~4.3. of \cite{ji2019polylogarithmic}. First, this martingale difference sequence is bounded almost surely as $R:=1/\ln{2} \geq Q(\W_t;\B_t) - Q_t(\W_t)$, simply because $0\leq -\ell'(z)\leq 1/\ln{2}$. The conditional variance can be bounded as:
\begin{align*}
V_t &:= \sum_{t=1}^{T}\E[ (Q(\W_t;\B_t)-Q_t(\W_t))^2 | \cS_{t-1}] \\
&= \sum_{t=1}^{T} Q(\W_t;\B_t)^2 -2Q(\W_t;\B_t)\E[Q_t(\W_t)| \cS_{t-1}] + \E[Q_t(\W_t)^2| \cS_{t-1}] \\
&\leq \sum_{t=1}^{T} \E[Q_t(\W_t)^2| \cS_{t-1}] \tag{$\E[Q_t(\W_t)| \cS_{t-1}] = Q(\W_t;\B_t)$} \\
&\leq \frac{1}{\ln{2}} \sum_{t=1}^{T} \E[Q_t(\W_t)| \cS_{t-1}] \tag{$0\leq Q_t(\W_t) \leq 1/\ln{2}$} \\
&= \frac{1}{\ln{2}}\sum_{t=1}^{T}Q(\W_t;\B_t)
\end{align*}
Thus, using Theorem~\ref{thm:martingale} with $R\leq 1/\ln{2}$ and $V_t \leq \sum_{t=1}^{T}Q(\W_t;\B_t)/\ln{2}$, we conclude that with probability at least $1-\delta$ it holds that
\begin{align*}
&\sum_{t=1}^{T} Q(\W_t;\B_t) - \sum_{t=1}^{T}Q_t(\W_t) \leq (e-2)\sum_{t=1}^{T}Q(\W_t;\B_t) + \frac{\ln{1/\delta}}{\ln{2}} \\
\implies &  \frac{1}{T}\sum_{t=1}^{T} Q(\W_t;\B_t) \leq \frac{4}{T}\sum_{t=1}^{T}Q_t(\W_t) + \frac{4\log(1/\delta)}{T}
\end{align*}
Plugging the above back in $\cR(\W_t;\B_t) \leq 2\ln{2} Q(\W_t;\B_t)$, and averaging over iterates we have:
\begin{equation*}
\frac{1}{T}\sum_{t=1}^{T}\cR(\W_t;\B_t) \leq \frac{16\ln{2}\lambda^2}{\eta T} + \frac{8\ln{2}\ln{1/\delta}}{T}
\end{equation*}
which completes the proof.
\end{proof}

\end{document}